\documentclass[nohyperref]{article}

\usepackage{microtype}
\usepackage[dvipsnames]{xcolor}
\usepackage{subcaption}
\usepackage{graphicx}

\usepackage{booktabs} % for professional tables

\usepackage{hyperref}

\DeclareUnicodeCharacter{2212}{-}

\newcommand{\half}{\frac{1}{2}}

\newcommand{\Abs}[1]{\left| #1 \right|}
\newcommand{\Set}[1]{\left\{ #1 \right\}}
\newcommand{\Brack}[1]{\left( #1 \right)}

\newcommand{\mc}[1]{\mathcal{#1}}

\usepackage[accepted]{icml2023}

\usepackage{amsmath}
\usepackage{amssymb}
\usepackage{mathtools}
\usepackage{amsthm}

\usepackage[capitalize,noabbrev]{cleveref}

\theoremstyle{plain}
\newtheorem{theorem}{Theorem}[section]
\newtheorem{proposition}[theorem]{Proposition}
\newtheorem{lemma}[theorem]{Lemma}

\theoremstyle{definition}

\theoremstyle{remark}

\usepackage[textsize=tiny]{todonotes}

\icmltitlerunning{Learning Hidden Markov Models When the Locations of Missing Observations are Unknown}

\begin{document}

\twocolumn[
\icmltitle{Learning Hidden Markov Models When the Locations of Missing Observations are Unknown}

\icmlsetsymbol{equal}{*}

\begin{icmlauthorlist}
\icmlauthor{Binyamin Perets}{equal,Technion}
\icmlauthor{Mark Kozdoba}{equal,Technion}
\icmlauthor{Shie Mannor}{Technion}

\end{icmlauthorlist}

\icmlaffiliation{Technion}{Technion Israel Institute of Technology, Haifa, Israel}

\icmlcorrespondingauthor{Binyamin Perets}{sbp67250@campus.technion.ac.il
}
\icmlkeywords{Hidden Markov Models, Missing Observations , Machine Learning, ICML, Statistics, Gibbs sampling, Sequential data analysis, Time series, Computational biology, Silent states HMM, Profile HMM}
\vskip 0.3in
]
\printAffiliationsAndNotice{\icmlEqualContribution} 
\begin{abstract}
The Hidden Markov Model (HMM) is one of the most widely used statistical models for sequential data analysis. One of the key reasons for this versatility is the ability of HMM to deal with missing data. However, standard HMM learning algorithms rely crucially on the assumption that the positions of the missing observations \emph{within the observation sequence} are known. In the natural sciences, where this assumption is often violated, special variants of HMM, commonly known as Silent-state HMMs (SHMMs), are used. Despite their widespread use, these algorithms strongly rely on specific structural assumptions of the underlying chain, such as acyclicity, thus limiting the applicability of these methods. Moreover, even in the acyclic case, it has been shown that these methods can lead to poor reconstruction. 
In this paper we consider the general problem of learning an HMM from data with unknown missing observation locations. We provide reconstruction algorithms that do not require any assumptions about the structure of the underlying chain, and can also be used with limited prior knowledge, unlike SHMM.
We evaluate and compare the algorithms in a variety of scenarios, measuring their reconstruction precision, and robustness under model miss-specification. Notably, we show that under proper specifications one can reconstruct the process dynamics as well as if the missing observations positions were known.
\end{abstract}

\section{Introduction}
Hidden Markov Models (HMMs),\cite{murphy2012machine}, are a well established and widely used method for modeling sequential data,  with applications in a variety of fields, such as speech recognition \cite{yu2016automatic} and economical time series \citep{hamilton2020time,kaufmann2019hidden}, to name a few. 
In all of the above applications, it is common to have \emph{missing observations}. That is, we assume that a certain system, after $N$ time steps, produces a full sequence of observations $U_i = (u_1, \ldots, u_N)$, but there is a \emph{known} subset of times,  $1\leq t_1, \ldots, t_m \leq N$, $m \leq N$, such that the observation values at these times, $\Set{u_{t_j}}_{j=1}^m$, are not known to us. This situation is illustrated in the first and second lines of Figure \ref{fig:illProcess}. We refer to this type of missing values as missing values with \emph{known locations} (or \emph{known gaps}, Interchangeably), since we know which of the values are missing.

\begin{figure}[h]
    \centering
    \begin{subfigure}[b]{0.42\textwidth}
        \includegraphics[scale=0.38]{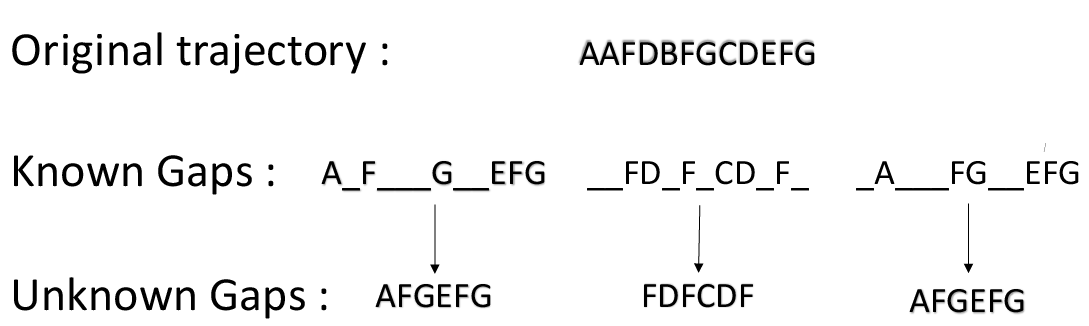}
        \vspace*{6mm}
        \caption{Omission Process Illustration}
        \label{fig:illProcess}
    \end{subfigure}
    \hspace{4pc}
    \begin{subfigure}[b]{0.45\textwidth}
        \includegraphics[scale=0.35]{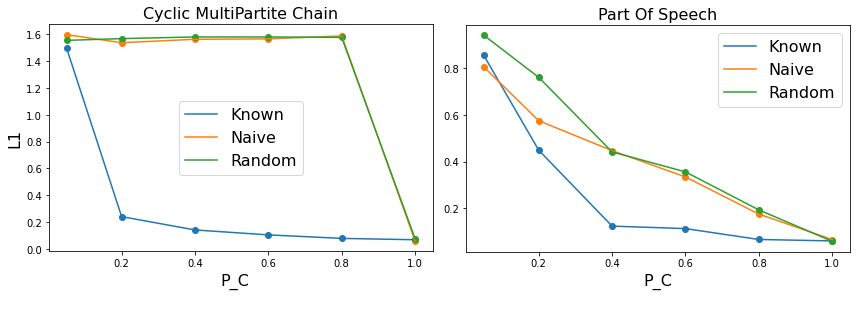}
        \caption{Ignoring the Gaps}
        \label{fig:NaiveIntro}
    \end{subfigure}
    \caption{
\textbf{(a)} The gaps location may alter the sequence considerably.
\textbf{(b)} The \emph{Naive} approach may perform poorly, highlighting the need for a more careful treatment of gaps. (s.d $<$ .062)}
    \label{fig:Intro}
\end{figure}

\subsection{Motivation and Previous Works} \label{Intro:part2}
It is commonly assumed that observations are provided with time stamps, and that a single time unit corresponds to a step of the underlying chain. Such assumptions naturally imply that when there are missing values their location is known, however, this assumption do not hold for many applications. For example consider the case of irregular time series \cite{IrrTimeSe}, also known as unevenly spaced time series, where the system is observed at varying time intervals that do not correspond to the steps of the underlying chain. For example, in the context of cyber-attack detection, it is often assumed that the attacker sequence of actions can be modeled as a Markov process \cite{Chadza2020-mz}. However, in practice, the observations of these malicious activities are irregular and incomplete. Another scenario is when the sampling clock is different from the underlying process clock. An example of this is the monitoring of sepsis patients, as the disease can progress in an unpredictable manner, resulting in inconsistent and possibly incomplete measurements of vital signs \cite{pmlr-v106-moor19a}.

One domain where missing observations in unknown locations is frequently encountered is computational biology, where data is often only ordered and lacks time stamps\citep{pmid29107741,liu2017reconstructing,orr2018learning}.
For that reason, Silent-state HMM were developed. In SHMM, some states are not associated with any observation, and the model can transition between these states without emitting any observation. These states are referred to as "silent" states. Notice that silent states are incorporated explicitly into the assumed chain structure. 
For example, consider the wide use of Profile hidden Markov models \cite{10.1093/bioinformatics/14.9.755} for DNA sequence searching, where the goal is to find optimal matches between a model input and a database of sequences \citep{HMMER1,Wheeler2013}. In this context, PHMMs are used to match the input model with sequences in the database that are assumed to come from the same model. However, it is also assumed that some states (nucleotides) may have been "deleted" through the course of evolution and are therefore not observable, or "silent". 

Despite their potential usefulness and popularity, the uses of PHMMs and SHMMs are limited to a small set of problems. Currently, the focus of PHMM use is on inference, where it is assumed that the transition matrix and omitting probabilities are known (known as the profile HMM with known profiles).
However, there is a significant need for many applications of SHMM to be able to directly learn the model parameters from data\citep{Wheeler2013,setty2019characterization, ye2019circular,orr2018learning,Pattabiraman2021-ol}. There are currently several limitations that have hindered this goal.
\textbf{First}, while the use of Expectation Maximization algorithms for learning the parameters of SHMM has been proposed, it has not yet been thoroughly studied and has not been shown to be effective in practice. In addition, it has been suggested in recent years that PHMMs are not identifiable \cite{Pattabiraman2021-ol}. In this paper, we address this limitation and examine the reasons for the failure of these algorithms, providing a new perspective on the identifiability issue. \textbf{Second}, while partial knowledge that is crucial for learning is often available, current frameworks are not flexible enough to incorporate this information.
\textbf{Most importantly}, SHMMs assume that the underlying chain is acyclic (i.e the underline chain forms a DAG), which is frequently violated in the field of computational biology \citep{ye2019circular,deconinck2021recent,LummertzdaRocha2018}. This assumption can be particularly problematic in areas such as pseudo-time analysis \cite{ye2019circular} in a variety of biological systems, \cite{campbell2018uncovering}, with particularly extensive applications in single-cell trajectory analysis which is crucial in understanding the dynamics of a cell (see for instance \citep{pmid29107741, van2020trajectory} and a survey \cite{deconinck2021recent}). 
In this type of applications, one is given a \emph{set} of observations, usually corresponding to information about the state of a cell. However, due to the nature of the data collection process, it is often impossible to obtain timestamps for each individual observation.

In the interest of providing a comprehensive overview, we would like to highlight the following previous works related to this topic: In \cite{IrrTimeSe}, the authors formally described the concept of irregular time series in HMMs, but did not propose an algorithm for reconstruction. \cite{NIPS2013_32b30a25} dealt with the sub-problem of flow network analysis where the positions of a few observed nodes are fixed and known in advance. The effort closest to our own, as far as we are aware, is the work of \cite{orr2018learning}, in which the authors developed an Expectation Maximization-based algorithm for reconstructing HMMs from data with similar properties, but under the restriction of acyclic chains and a strict assumption about the underlying Markov process.

\subsection{Missing Observations In Unknown Locations and The Need For Special Methods}
\label{Intro:MissingObsIntro}
Let $O= \Brack{o_1, \ldots, o_k}$ be an \emph{order} set of observations, that is, we assume that a system passed through some sequence of states $S = \Brack{s_1, \ldots, s_N}$, \emph{some} of which have generated the observations $O$. This situation is illustrated in the third line of Figure \ref{fig:illProcess}, where different subsequences correspond to different possible values of $O$. Note in particular that the same $O$ can result from different deletions of the full observation sequence. In this situation, we do not apriori know the positioning of the observations $O$ (and, equivalently, of the observations missing from $O$) with respect to the state sequence $S$, and therefore we refer to such observations as observations with \emph{unknown gaps}.
The goal of the reconstruction (or learning) algorithm is to recover the underlying dynamics of the state sequence $S$ from the observations $O$. 

It is natural to ask whether one indeed needs to treat the unknown gaps situation specially. For example, a possible approach to reconstruction, termed \emph{``Naive"} in this paper, is to simply ignore the missing observations and to reconstruct from sequences $O$ as if these were the full observation sequences $U$. See for instance \citep{liu2017reconstructing,setty2019characterization}. We now show that the \emph{Naive} method can indeed introduce extremely large errors into estimation. First, we denote $\Psi(s)$ to be the probability of observation emitting from state $s$ to be omitted (see Section \ref{sec:background}). For a fixed benchmark HMM and a fixed $\Psi(s)= 1-p_c : \forall s$, we generate synthetic observation trajectories with gaps based on $p_c$, and compare three reconstruction methods: (a) The \emph{Naive} approach, where we ignore the missing observations, (b) The Random approach, where we use the standard reconstruction algorithm with \emph{known} gaps, but the gap locations given to the algorithm are chosen at random, and (c) The ``Known'' reconstruction, where the gap locations provided to the algorithm are the \emph{true} gap locations. This algorithm is used as an \textbf{ideal benchmark}, as its performance is the best one can hope for, for any unknown gaps algorithm.
The experiment is repeated for different benchmark HMM (see Section \ref{sec:experiments} for full details), and for a range of values of $p_c$. For each run, we measure the quality of the reconstruction by the $L_1$ distance between the reconstructed and the ground truth transition matrices.

As shown in Figure \ref{fig:NaiveIntro}, for intermediate values such as $p_c = 0.5$ the performance of the \emph{Naive} and random approaches are poor, and for some (MultiPartite HMM) this is especially pronounced, with error of $1.6$ even at $p_c = 0.8$ (the maximal error of $L_1$ is $2$).
This experiment emphasizes the need for reconstruction algorithms that model the unknown gap locations much more carefully than the \emph{Naive} and ``random'' methods. Appendix \ref{Supp:wrognW} contains additional benchmark HMMs.

\subsection{Paper Summary and Contributions}\label{sec:SummeryContribution}
This paper aims to propose a comprehensive solution for learning HMMs in scenarios where observations are missing in unknown locations. Those solutions do not impose restrictive assumptions on the underlying chain, with the exception of the "Naive" method, which is shown to be ineffective. We demonstrate the effectiveness of our approach through various methods, which correspond to varying levels of prior knowledge about the process. These methods enable reconstruction in cases where prior knowledge is partially known, which is not possible with previous methods. The rest of the paper is organized as follows:

\textbf{Model definition}. Section \ref{sec:background} defines the HMMOP (HMM with Omission Process) model. In this model, a set $U$ of full sequences of observations  $U_i = (u_{i1},\ldots,u_{iN})$ is generated by a standard HMM. Then, the sequence $O_i$ is formed by deleting (equivalently,  omitting) entries $u_{ij} \in U_i$ at random with a probability $\Psi$ which depends on the state, independently of the other $u_{i'j'}$. This type of state dependent omitting process in known as \textbf{Non- Ignorable Omitting Process}. Section \ref{sec:nonign} examines the non-ignorable setup for missing observations, while Section \ref{ssec:noiid} emphasizes the significance of taking $\Psi(\cdot)$ into account for accurate reconstruction.\\
\textbf{Analytical Analysis}.
Section \ref{sec:analytic} presents an analytical analysis for the case when the probability of missing observations is constant across all states, $\Psi(s) = 1-p_c \forall s$. This analysis highlights the limitations of previous attempts to reconstruct HMMs with missing values using PHMMs.\\
\textbf{Reconstruction methods}. Section \ref{sec:gibbs} introduces two approaches to reconstructing HMMOPs. One approach assumes that for the ignorable setup, only the observed trajectories $\mathcal{O} = \Set{O_i}$ and the length of the original sequence are known. The second approach assumes that only the omitting probabilities ($\Psi$) are known. Additionally, we present a method for imputing $\Psi$ when only partial knowledge of it is available.\\
\textbf{Experiments and evaluations}. Section \ref{sec:experiments} evaluates the unknown gaps reconstruction algorithms on a number of synthetic and semi-synthetic benchmarks, where it is possible to compare the reconstructed transition matrices to the ground truth matrices that generated the data. 
The experiments are performed for a range of $\Psi$, allowing us to study the effect of different percentages of missing observations on the performance. It is worth noting that given $\Psi$ is known, our reconstruction method (term ``Gaps sampler") \textbf{match the performance of an ideal benchmark for which the locations are known}. This result is particularly remarkable, as it demonstrates that our method is able to achieve the best possible performance. Indeed, observe that the reconstruction algorithm only takes $\Psi$ (or a single number $p_c$ for the ignorable case), as information about missing samples. Given an observation sequence $O_i$, there is an exponential number of possible placings of gaps within that sequence, and we know from the experiment described in Figure \ref{fig:NaiveIntro} that a random placement of gaps yield poor results. Yet, the algorithms manage either to get close to, or match the performance of the sampler that knows the locations of the missing values. To the best of our knowledge, this is the first time such results have been achieved. \\ 
\textbf{Experiments and evaluations under miss-specifications}. In Section \ref{sec:pcMiss}, we evaluate the robustness of the algorithms with respect to various \emph{model miss-specifications}. First we study the effect of providing a perturbed $\Psi$ to the algorithm. Second, we study the ability to reconstruct when $\Psi$ is unknown for some states. Third, we evaluate other miss-specifications w.r.t the omitting process: We tested the performance for a non-constant $\Psi$ (that is, changing per sentence) and a scenario where the omitting process is a Markov process. For all of the above miss-specifications we show that the results are quite stable. \\
\textbf{Code}. To the best of our knowledge, our implementation is the first publicly available Gibbs sampling-based HMM learning implementation for Python, and the first to handle non-ignorable missing observations in general. The code is provided in the supplementary material and will be made publicly available with the final version of the paper.

\section{Background and Model Definition}
\label{sec:background} 
A Hidden Markov Model  $(\mathbb{X},T,\Theta)$ is defined by a finite state space $\mathbb{X}$ and a transition probability matrix $T \in R^{\Abs{\mathbb{X}} \times \Abs{\mathbb{X}}}$ which defines a Markov chain on $\mathbb{X}$.  
The observations of the model take values in a set $\mc{E}$, and for every state $X \in \mathbb{X}$, $\Theta_X$ is a distribution on $\mc{E}$ corresponding to $X$. 
Let $X_i = [X_{i0},\ldots,X_{iN}]$ be a sequence of states, and let $U_i = [u_{i0},\ldots,u_{iK}]$ be a sequence of observations (also referred to as \emph{sentence}), 
$u_{ij} \in \mc{E}$, $K\leq N$.\\
The sequence $W_i$ encodes the locations of the non-missing observations. Define $W_i = [w_0,\ldots,w_K] $ as a sequence of indexes $w_k \in \{0,..,N\}$,$w_{k}<w_{k+1}$, such that $w_k \in W_i$ if observation $u_{ik}$
was generated at time $w_k$. 
We denote by $U_i^{W_i} = [u_{iw_0},\ldots,u_{iw_K}]$ the restriction of $U_i$ to $W_i$.

\textbf{HMMOP - Unknown Gaps Model Definition}.\label{ProblemSetup}\\
Let $M' = (\mathbb{X},T,\Theta)$ be an HMM. Define a random mapping $\Phi_{\Psi}(\cdot)$ of a full observation sequence $U_i$ to a partial observation sequence $O_i$ as follows: 
Given a full sequence of states, $X_i$, corresponding to an observation sequence, $U_i$, an indicator vector, $C_i = [c_{i0},\ldots,c_{iN}]$ is sampled where $c_{ij}$ are independent Bernoulli variables with success probability $\Psi(X_{ij})$. The set $W_i = \Set{j \vert c_{ij} = 1}$ is defined, and the partial observation sequence, $O_i$, is set as $O_i = \Phi_{\Psi}(U_i) := U_i^{W_i}$.
Then HMMOP is defined to be the generative model where a state sequence $X_i$ and a full observation sequences $U_i$ are generated by an HMM, and the corresponding observation sequences $O_i$ are then produced as $O_i = \Phi_{\Psi}(U_i)$. The joint distribution of the HMMOP is :
\begin{equation*}
\begin{split}
    &P(O_i, W_i, X_i|\Theta) = P(X_0|\Theta) P(O_0|X_0) \ \cdot\\ &\prod_{t=1}^{N} T_{X_t,X_{t-1}} \cdot \begin{cases}
    (1-\Psi(X_t)) \cdot P(O_t|X_t,\Theta) ,  \text{if}\ t \in W_i \\
  \Psi(X_t),  \text{otherwise}
\end{cases}
\end{split}
\end{equation*}

\subsection{Non Ignorable Omitting Process} \label{sec:nonign}
Non Ignorable omitting process in HMMs, also known as State-Dependent Missingness (SDM), was first introduced by \cite{doi:10.1080/03610918.2011.581778} and later developed by \cite{https://doi.org/10.48550/arxiv.2109.02770}.
Reconstruction with non-ignorable missing observations refers to the problem of reconstructing HMMs when the probability of an observation to be missing depends on the state of the system, in contrast to the ignorable case where $\Psi(s)$ is constant and simply translate to the percentage of missing observations in the data. Reconstruction of dynamical process with SDM (not necessarily HMMs) is considered a challenging problem and it is an active area of research. In this paper we generally consider the non-ignorable case and we carefully examine the impact of the SDM assumption on our reconstruction results. Notably, we show that we can still achieve ideal results for known omitting probabilities ($\Psi$) even in the SDM case. Section \ref{sec:experiments} presents comprehensive evaluation of the ignorable case.

\subsection{Analytical Analysis}\label{sec:analytic}
For analytical analysis we consider the case of ignorable missing observations, that is, $\Psi(s) = 1-p_c \forall s$. More, let us first discuss the \emph{non-hidden} case. That is, we assume for the moment that the observation given a state $X \in \mathbb{X}$ is the state $X$ itself. Note that the unknown gaps reconstruction is interesting even in this simpler case. \\
% Lets define $T'$ as the transition matrix learned directly from the missing sentences $\mathcal{O}$ using standard methods (denoted before as the ``Naive" approach). We denote $T$ as the latent transition matrix which we aim to reconstruct. \\

\begin{proposition}\label{MMM}
Let $M' = (\mathbb{X},T,\Theta)$ be an non-hidden Markov model. $\mathbb{I}$ the identity matrix. Then the sequence of observations $O_i = \Phi_{\Psi = p_c}(U_i)$ is a Markov chain with transition matrix :
\begin{equation}
\label{eq8}
T_r =   p_c \cdot T \cdot [\mathbb{I} - (1-p_c) \cdot T]^{-1}.
\end{equation}
\end{proposition}

\begin{proof}
By definition, in the non-hidden case the sequence $U_i$ coincides with the state sequence $X_i$, and $O_i$ is a sub-sequence of $U_i$. Note that the waiting time $d$ between two occurrences of $1$ in $C_i$ is geometrically distributed with mean $p_c$, which implies that 
\begin{equation}
\label{eq:T_r_sum}
T_r = \sum_{d=1}^{\infty} p_c (1-p_c)^{d-1} T^{d}.
\end{equation}
For every stochastic matrix $T$ we have $\|T\|\leq 1$, \cite{UpperBoundsStochastic}, where $\|\cdot\|$ is the operator norm, $\|(1-p_c) \cdot T\| < 1$, and thus the series in \eqref{eq:T_r_sum} are summable, with sum given by \eqref{eq8}.
\end{proof}
Now, note that given the partial observations $O_i$, we can directly learn the transition matrix $T_r$ for $O_i$ by counting the co-occurrences of the states in $O_i$. This is the instantiation of the \emph{Naive} approach. Next, observe that if $p_c$ is known, using \eqref{eq8} we can \emph{invert} the omission process to obtain $T$: 
\begin{equation} \label{eq:eqbackward}
    T =  [\mathbb{I} \cdot p_c + (1-p_c) \cdot T_{r}]^{-1} \cdot T_{r}.
\end{equation}
We refer to this as the \emph{backward transformation}.

As we can see, $T$ is a transition matrix of a Markov model for any $p_c \in (0, 1]$. Therefore, given a set of missing observations $O$, it is not possible to infer any "original"  T, but rather only a backward version of T that is conditioned on a specific value of $p_c$.
As Supplementary material \ref{suppFB} shows, the difference between reconstructed T for different $p_c$ might be significant.\\
 The current approach to reconstructing Profile-HMMs involves guessing a random $p_c$, then using an Expectation-Maximization (EM) procedure to learn $T$ given the guessed $p_c$, and vice-versa. However, as we demonstrated, guessing the $p_c$ is equivalent to randomly selecting a backward transformation of $T$, which has no practical meaning. This we believe explain the lack of successful PHMMs reconstruction studies, despite its importance.
Detailed evaluations of \eqref{eq8},\eqref{eq:eqbackward}, together with additional results and characteristics of the backward transformation, are given in Supplementary Material Section \ref{suppFB}

Note that the ``backward transformation" can be seen as a reconstruction method for the case of ignorable missing observations. \textbf{For the proper hidden case}, we can apply any standard HMM learning algorithm to learn $T_r$ using the \emph{Naive} approach from $\mathcal{O}$. Then, we can use \eqref{eq:eqbackward} in the same way as in the non-hidden case to obtain $T$. Since this approach requires some HMM learning, which is typically non-analytic, we refer to this approach as semi-analytic. We evaluate the results of this method in compare to the other methods in Evaluation Section \ref{sec:experiments}.

\section{Reconstruction Methods}
\label{sec:gibbs}
\begin{figure}[t]
    \centering
    \begin{subfigure}[b]{0.42\textwidth}
        \includegraphics[scale=0.28]{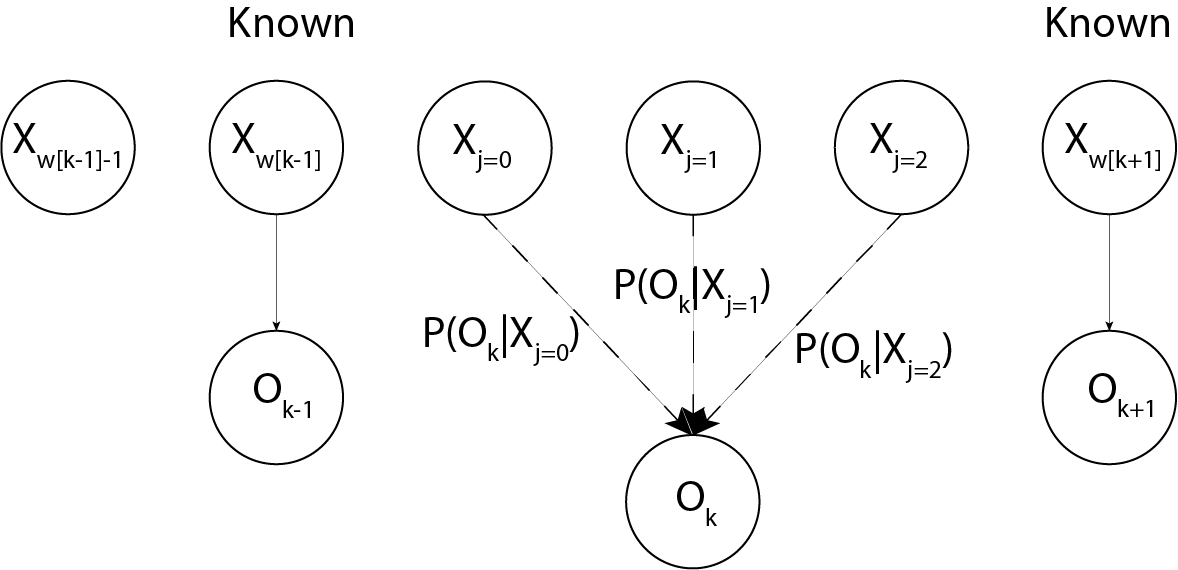}
        \vspace*{2mm}
        \caption{Illustration of W Sampling given $\Psi = p_c$.}
        \label{fig:illusWSample}
    \end{subfigure}
    \hspace{3pc}
    \begin{subfigure}[b]{0.42\textwidth}
        \includegraphics[scale=0.28]{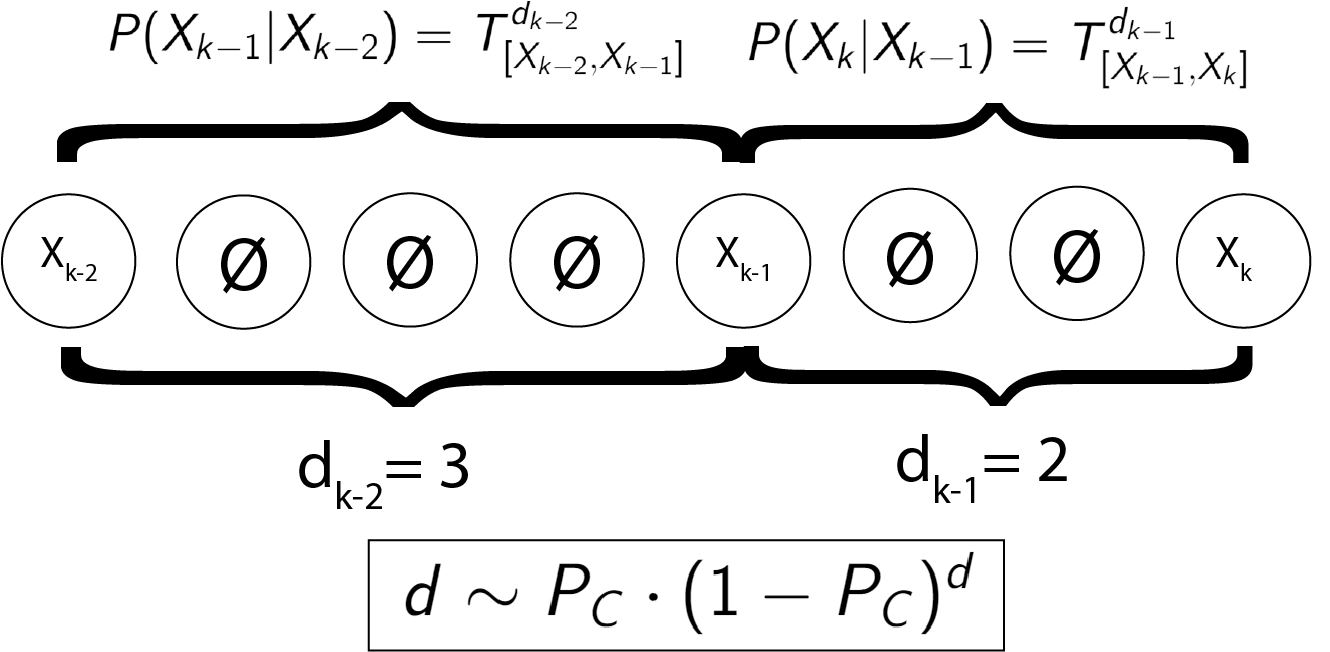}
        \caption{Illustration of HMMOP given $\Psi = p_c$.}
        \label{fig:IllKnownPC}
    \end{subfigure}
    
    \caption{ Reconstruction methods. }
\label{fig:sampleIll}
\end{figure}

In this section, we present the reconstruction methods. As discussed in Sections \ref{sec:SummeryContribution} and \ref{sec:analytic}, these methods rely on different levels of knowledge about the underlying latent dynamic process or the omitting process. \textbf{The first method}, named the "Matching Sampler," requires knowledge of the full unknown sequence length (N) as input. For the case of ignorable missing observations, this method does not require any additional information. However, in the case of non-ignorable missing observations, knowledge of the omitting probabilities ($\Psi$) is required. We show that even if this information is only partially known, it can be inferred as part of the algorithm and the method is highly robust to misspecifications regarding $\Psi$. \textbf{The second method}, named the "Gaps Sampler," only requires knowledge of $\Psi$ as input. Although it does not have access to information about N, the "Gaps Sampler" is also somewhat robust to misspecifications regarding $\Psi$.
Often, knowledge about $N$ can be derived from prior knowledge on the underlying process $M'$, for instance, when there are distinctive "start" and "end" states with knowledge about the latent process rate \citep{Saelens2019,Herring2018}. On the other hand, knowledge about $\Psi$ exists for cases in which prior knowledge on the sampling method ($\Phi_{\Psi}$) exists \cite{LummertzdaRocha2018}. In addition, in cases where the reconstruction is a first step for a prediction task \cite{10.1371/journal.pone.0221245}, $\Psi$ can be partially inferred as an hyper parameter using the labeled data (Supplementary Material Section \ref{Supp:Inference} presents a prediction algorithm for HMMOP).

\textbf{Gibbs sampling.} Both of the methods relay on the Gibbs sampler \citep{gelman2013bayesian,10.1214/08-BA326}. A Gibbs sampler samples the HMM parameters (i.e., $T$ and $\Theta$) and the latent state sequences $\mathcal{X}$ conditioned on $\mathcal{O}$. More specifically, one interchangeably samples $P(X_i|T,\Theta,\mathcal{O})$, and then samples $P(\Theta|X_i,\mathcal{O})$ and $P(T|X_i)$.
Here we extend this procedure by incorporating the uncertainty of the unknown gaps into the sampling process using an additional set of latent helper variables.
Let $ \mathcal{X} = \Set{X_i}$ be the set of all latent states sequences as describe above, and let $O_i\in \mathcal{O}$ be the observed sequences corresponding to $X_i$.

\subsection{The Matching Sampler}\label{MatchingSampler}
The Matching Sampler builds upon $\mathcal{W} = \Set{W_i}$ (see Section \ref{sec:background}) as an helper set of latent variables, hence the challenge mainly lies in sampling the helper variables $P(W_i|X_i, N_i, O_i, T, \Theta)$. To address this challenge, we first note that the conditional distribution of $W_i$ given $X_i$ is independent of $T$. Therefore, our problem reduces to sampling $P(W_i = [w_{i0},\dots, w_{iK}]|\Theta,X_i,O_i,N_i,\Psi).$\\
For $k\leq K$, let $W^k$ denote the sequence $W$ with $w_k$ omitted. Also let  $I_{[w_{k-1},w_{k+1}]}(w)$ be the indciator function.  
The sampling of $W_i$ will be done using its own Gibbs sampler. That is, instead of sampling $W_i$, we iteratively sample its components,$w_{ik}$, conditioned on the rest of the matching $P(w_{ik}|W_i^k,\Theta,X_i,O_i,\Psi).$
The Markov property of the sequence $X$ implies that $w_{ik}$ depends on the sequences $W_i^k,O_i,X_i$ only through ``local variables'' at $k$,  that is, it depends on $X_{i,w_{i,k-1}}$ up to $X_{i,w_{i,k+1}}$, and on the single observation $O_{ik}$ :
\begin{equation} \label{PWk_2}
\begin{split}
    &P(w_{ik}|W_i^k,\Theta,X_i ,O_i,\Psi) \propto \\ &P(O_{k}|  X_{w_k}) \cdot (1-\Psi(X_{w_k})) 
  \cdot I_{[w_{k-1},w_{k+1}]}(w_k)  
\end{split}
\end{equation}
This situation is illustrated in Figure \ref{fig:illusWSample}.

One can see a similarity between $W_i$ to the ``wrapping" vector from the known Stochastic Dynamic Time Wrapping (DTW) \cite{doi:10.1080/03772063.1988.11436710}. Notice that while variations of SDTW can infer optimal W from $P(W|\ldots)$, our challenge is \textbf{sampling} $P(W|\ldots)$. Moreover, not every state in $X_i$ is part of $W_i$ as the SDTW algorithm requires.  

In practice we use the Metropolis-Hasting \cite{doi:10.1063/1.1699114} for better convergence. Full details of the derivation of \eqref{PWk_2} and the complete algorithm, including the sampling of $T$, $X$, and $W$, as well as initialization considerations, can be found in Section \ref{GibbsSamplerKnownN} of the appendix.

\begin{figure*}[t]
\centering
\begin{subfigure}[b]{0.32\textwidth}
         \includegraphics[width=\textwidth]{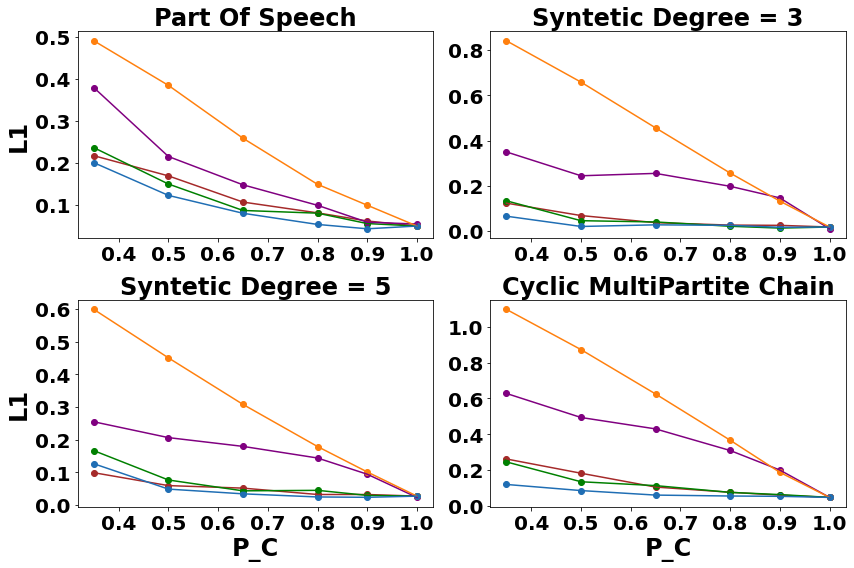}
         \caption{Reconstruction for ignorable setup.}
         \label{fig:knownN}
\end{subfigure}
\hfill
\begin{subfigure}[b]{0.32\textwidth}
         \includegraphics[width=\textwidth]{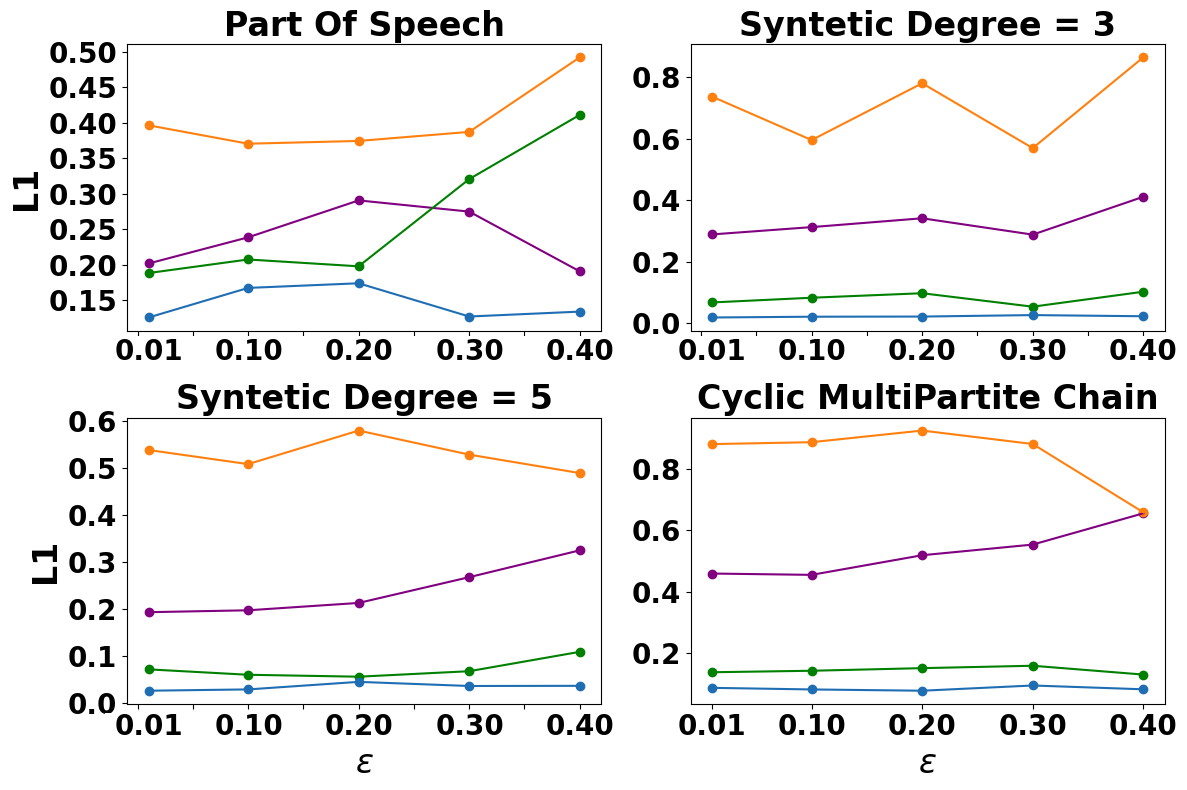}
         \caption{Reconstruction for nonignorable setup.}
         \label{fig:nonIgnknownN}
\end{subfigure}
\hfill
\begin{subfigure}[b]{0.32\textwidth}
          \includegraphics[width=\textwidth]{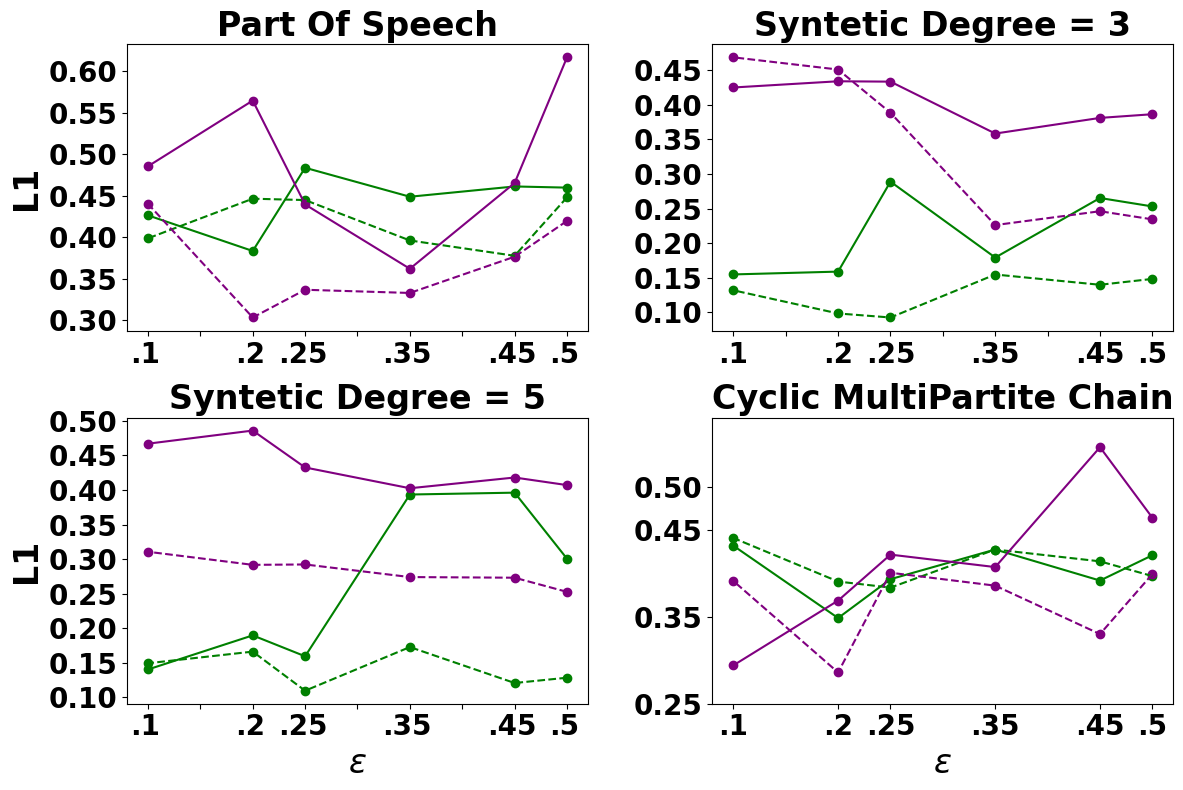}
          \caption{The impact of $\Psi$ imputation.}
    \label{fig:imputation}
\end{subfigure}
\caption{\textbf{(a, c)} Both Samplers significantly improves on the \emph{Naive} method. While the "Gaps" is better. \textbf{(b)} The Gap sampler converge to a better placement, faster.}
\label{fig:knownstuff}
\end{figure*}

\subsection{The Gaps Sampler}\label{sec:Gap}
The Gaps Sampler presents a different representation for the missing observations locations by introducing variables $d = [d_0, \ldots, d_K]$. Those variables represent the number of omitted observations between any two observed observations, which we denote as gaps. Figure \ref{fig:IllKnownPC} illustrates the Gaps sampler for the ignorable case for which $\Psi(\cdot) = p_c$. Our goal is then sample $P(d_k |X_{k,k+1})$. Lets define $s = [s_0,\ldots,s_{d-1}] \in S^d$ as the set of all $d$ long states trajectories, also, lets assume $d < \mathcal{S}$ for a predefined $\mathcal{S}$. We say that a sequence is \emph{fully omitted} if we didn't observed any of the sequence states, we denote the case of a fully omitted $s$ as $\Bar{s}$. Given $\Phi_{\Psi}$ independent, $P(\Bar{s}|s) = \prod_{i=0}^{d-1} \Psi(s_i) $.
By definition, $P(d_k |X_{k,k+1})$ is the probability for $s \in S^d$ been fully omitted: 
\begin{multline*}
 P(d_k |X_{k,k+1}) =\\ 
 \frac{\sum\limits_{s \in S^d} T_{[X_k,s_0]} \cdot T_{[s_{d-1}, X_{k+1}]} \cdot \Psi(s_0) \cdot \prod\limits_{i=1}^{d-1} T_{[s_{i-1}, s_{i}]} \cdot \Psi(s_i)
  }{\sum\limits_{\tau = 0}^{\mathcal{S}} \sum\limits_{s' \in S^{\tau}} P(s'|X_{k,k+1})\cdot P(\Bar{s'}|s')} \\  \propto \sum_{s \in S^d} 
  T_{[X_k,s_0]} \cdot T_{[s_{d-1}, X_{k+1}]} \cdot \Psi(s_0) \cdot \prod\limits_{i=1}^{d-1} T_{[s_{i-1}, s_{i}]} \cdot \Psi(s_i)
\end{multline*}
And given that summing over $s \in S^d$ is not feasible, we suggest the following formulation: 
\begin{multline*}
=\sum\limits_{x \in \mathcal{X}} \sum\limits_{j \in \mathcal{X}} T_{[j,x]} \cdot \Psi(x) \cdot P(d-1|X_{k,k+1},s_{d-2}=j) \\
P(d-1|(X_{k, k+1})) = \sum\limits_{j \in \mathcal{X}} P(d-1|X_{k,k+1},s_{d-2} = j)  
\end{multline*}

This form the \emph{forward algorithm} \cite{18626} with $P(d|X_{k,k+1},s_{d-1} = j)$ as the forwarding elements.
Notice that $P(\ldots, d-1)$ are intermediate steps for calculating $p(d)$, hence the computational complexity of $P(1), \ldots, P(\mathcal{S})$ and $P(\mathcal{S})$ are equal.
Full details of the Gaps sampler are given in Section \ref{supp:samplingd} of the appendix.

\subsubsection{Sampling $\Psi$}\label{subsec:SamplingPsi}
Both methods include the ability to sample and impute $\Psi$ for cases where $\Psi$ is partially known. Note that $\Psi(s)$ follows a Bernoulli distribution, so the conjugate prior for $\Psi(s)$ for each $s$ is the Beta distribution. Hence, the posterior distribution of $\Psi(s)$ given $X$ and $W$ is $\Psi(s)|X,W \sim Beta \left(\mu(s \in X^W), \mu(s \in X) - \mu(s \in X^W)\right)$ where $\mu(\cdot)$ is the counting measure. Section \ref{sec:experiments} extensively evaluate the ability of our methods to impute partial $\Psi$.

\section{Experiments}
\label{sec:experiments}

This section evaluates the performance of the following reconstruction methods:
1. The Matching sampler (\textbf{\textcolor{Purple}{Purple}});
2. The Gaps sampler (\textbf{\textcolor{Green}{Green}}); 
3. The \emph{Naive} reconstruction as described in Section \ref{Intro:MissingObsIntro} (\textbf{\textcolor{BurntOrange}{Orange}}); And,
4. The ``ideal benchmark'' (\textbf{\textcolor{Blue}{Blue}}) where the true gaps locations are known to the sampler as described in Section \ref{Intro:part2}; Additionally, for the ignorable case, we evaluate the semi analytical approach (\textbf{\textcolor{Brown}{Brown}}) which is a byproduct of the analytical analyses; \\
The following models were used to generate the data (full details are given in Supplementary Material Section \ref{Supp:EvaluationsModel}): 
1,2. \textbf{``Synthetic Degree$\{d\}$" } A synthetic chain on 10 states, where each state has $d$ outgoing  transitions, chosen at random; 
3. \textbf{``Cyclic MultiPartite"} Synthetic chain with 25 states. The states are divided into 5 groups ranging from 1 to 5, transitions are only possible between states of consecutive groups and in a cyclic way. For example, a state from group 4 only moves to a state in group 0, etc; 
4. \textbf{``Part Of Speech"} process. Transitions and emissions (part of speech and words respectively) probabilities were extracted from the Brown NLP corpus \cite{BrownDS}. The trajectories were then sampled based on those parameters. \\
With the exceptions of the ``Part Of Speech" chains, state emissions are distributed with Normal distribution $N(\mu_i,0.1)$ for $\mu_i \in [0,n_{state}-1]$;
Unless specified otherwise, 1500 sentence($U$) of length 80($N$) were sampled. For all figures, Y-axis is the $L_1$ distance between the reconstructed and the ground truth transition matrices. Supplementary Material Section \ref{Supp:SD} presents the standard deviations (s.d) of this section results. 

Note that transitions are of paramount importance in biological settings, as the transition matrix often model the biological mechanism behind the process in question \citep{shokoohi2019hidden, ye2019circular, yoon2009hidden}. Furthermore, in these instances, the states and the emissions probabilities are commonly known \citep{setty2019characterization, Liu2017-yh, LummertzdaRocha2018}.
In our experimental setup, we adopted parameters that bear close resemblance to those observed in biological systems. For instance, in numerous genomics applications employing HMMs, the state count is commonly 6 (comprising four observable and two hidden states) or 21, as elaborated in \cite{yoon2009hidden}. Moreover, it's worth highlighting that the transition matrix in such cases is often sparse. In addition, in Cell Trajectory Analyses (CTA), the state count is usually determined by user's discretion but typically does not exceed 20, as validated by \cite{shokoohi2019hidden, ye2019circular}, where 12 states were deployed. In \cite{Saelens2019}, one of the most exhaustive CTA method comparisons to date, the maximum state count was set at 20. CTAs often can be seen as a branching process that incorporates "local" cycles and a slower "global" cycle, a presumption that is integrated into our bipartite model.

\subsection{Reconstruction Under Correct Specification}\label{exp:KnownNPC}
This section evaluates the reconstruction performance of the methods with no miss-specification. Figure \ref{fig:knownN} presents the reconstruction results for the ignorable case, that is, only a single number is provided to the methods which is the true percentage of missing observations (or the true trajectory length $N$ for the Matching sampler). In this experiment, we pick varying $p_c$'s (X-axis), i.e. $\Psi = 1-p_c$ percentage of observations are deleted. 
As the figure shows, the performance of all the algorithms is significantly better than the \emph{Naive} algorithm, which for most cases, is the only applicable one (\ref{Intro:part2}). Moreover, for the Gaps sampler, the reconstruction performance \textbf{matches that of the ideal benchmark} (\ref{Intro:part2}). As discussed in Section \ref{sec:SummeryContribution}, this result is remarkable, due to the exponential number of possible placements $W$. 

To gain a better understanding of the performance distinctions between the samplers, we will conduct a more in-depth examination of them. To that end, we fix a ground-truth $T$, and learn the placement parameters $W$ (or, equivalently,  $\mathbf{d}$ ). That is, $T$ is fixed, and in a single iteration of the sampler, we sample $P(X \vert W,T,O)$ and then $P(W \vert X,T,O)$. The performance was measured by evaluating the average log likelihood of the trajectories, $\frac{1}{|\mathcal{O}|} \sum_{i}^{\mathcal{O}} log \left( P(O_i,X_i \vert T,W_i) \right)$ after each iteration of the sampler. We have also varied the number of times $P(W|X,O)$ is sampled, that is, for a fixed $X$, the matching sampler had from 30 to 300 steps to find the ``right'' $W$, before the next $X$ was sampled. For the Gaps sampler, only one step of $W$ sampling is required by design per $X$ sample.
As shown in Figures \ref{fig:convergenceSpeed}, while on the degree $d=5$ the performance is comparable, on the Cyclic MultiChain data the Gaps sampler finds more likely trajectories dramatically faster.

\begin{figure}[h]
\begin{center}
          \includegraphics[width=0.4\textwidth]{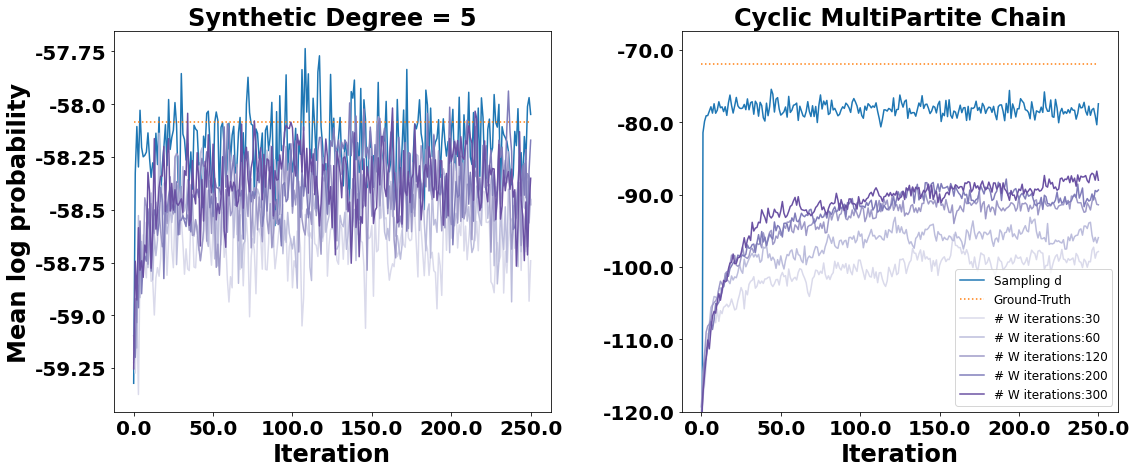}
          \vspace*{-3mm}
          \caption{Convergence of the placements parameters.  }
    \label{fig:convergenceSpeed}
\end{center}
\end{figure}

Figure \ref{fig:nonIgnknownN} presents the evaluation results for the non-ignorable case. In this experiment, for each $s$ we sample $\Psi(s) \sim Uniform[0.5 - \epsilon , 0.5 + \epsilon]$ for varying $\epsilon$ (X-axis), which control the variance of $\Psi(s)$. After omitting observations according to $\Psi(s)$, we provide it to the methods.
As the figure shows, the results are consistent with those of Figure \ref{fig:knownN} even for drastic variance in $\Psi(s)$. Previous works has demonstrated a decrease in performance when dealing with non-ignorable missing observations \cite{https://doi.org/10.48550/arxiv.2109.02770}, but our methods appear to be effective in this case, likely due to the use of Gibbs sampling which impute missing observations as an intermediate step.

Figure \ref{fig:imputation} evaluates the performance for when only limited knowledge of $\Psi$ is available. This experiment highlights a key advantages of sampling-based methods, which can infer missing values of $\Psi$. 
Specifically, we randomly generate a single $\Psi(s) \sim Uniform[0.35 , 0.65]$, and then delete a proportion, $\epsilon$, of entries. The partial $\Psi(s)$ is then provided to the methods. In the "Sampling imputation" case (dashed lines), the missing entries are inferred using the sampling procedure described in Section \ref{subsec:SamplingPsi}, while in the "Random imputation" case, they are filled in with random values.\\
As shown in Figure \ref{fig:imputation}, the imputation step leads to a significant improvement in performance. The Matching Sampler, which is provided with knowledge of $N$, demonstrates improved results across all benchmarks and levels of $\epsilon$. Additionally, the Gaps Sampler, which does not have access to this information, also shows notable improvement for some benchmarks. These results are noteworthy, as knowledge of $N$ is common in many applications of PHMMs, and obtaining prior knowledge about $\Psi$ can be costly or difficult to infer. Furthermore, the fact that the Gaps Sampler does not require knowledge of $N$, making this result even more impressive.

\subsection{Robustness Under Misspecification}\label{sec:pcMiss}

\begin{figure*}[t]
\centering
\begin{subfigure}{0.32\textwidth}
  \includegraphics[width=\textwidth]{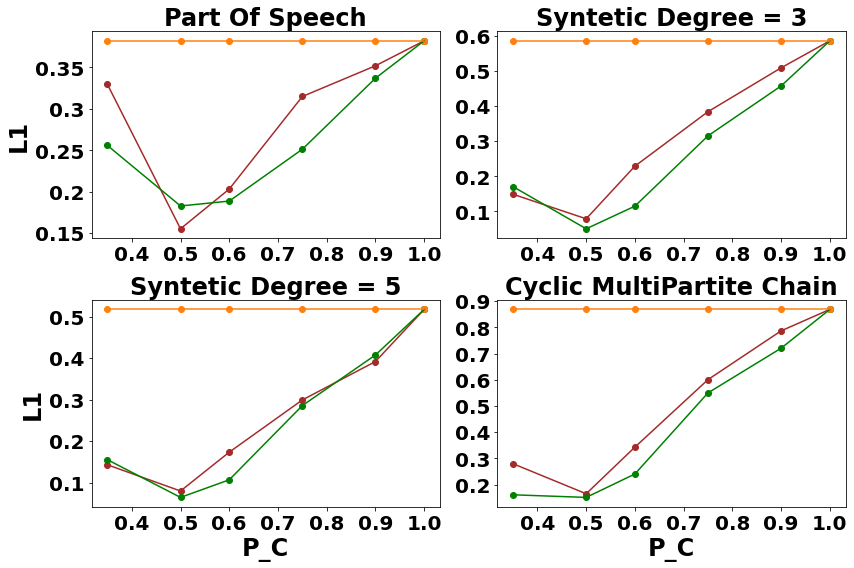}
  \caption{Wrong $\mathbf{p_c}$.}
    \label{fig:wrongPC}
\end{subfigure}
\hfill
\begin{subfigure}{0.32\textwidth}
    \includegraphics[width=\textwidth]{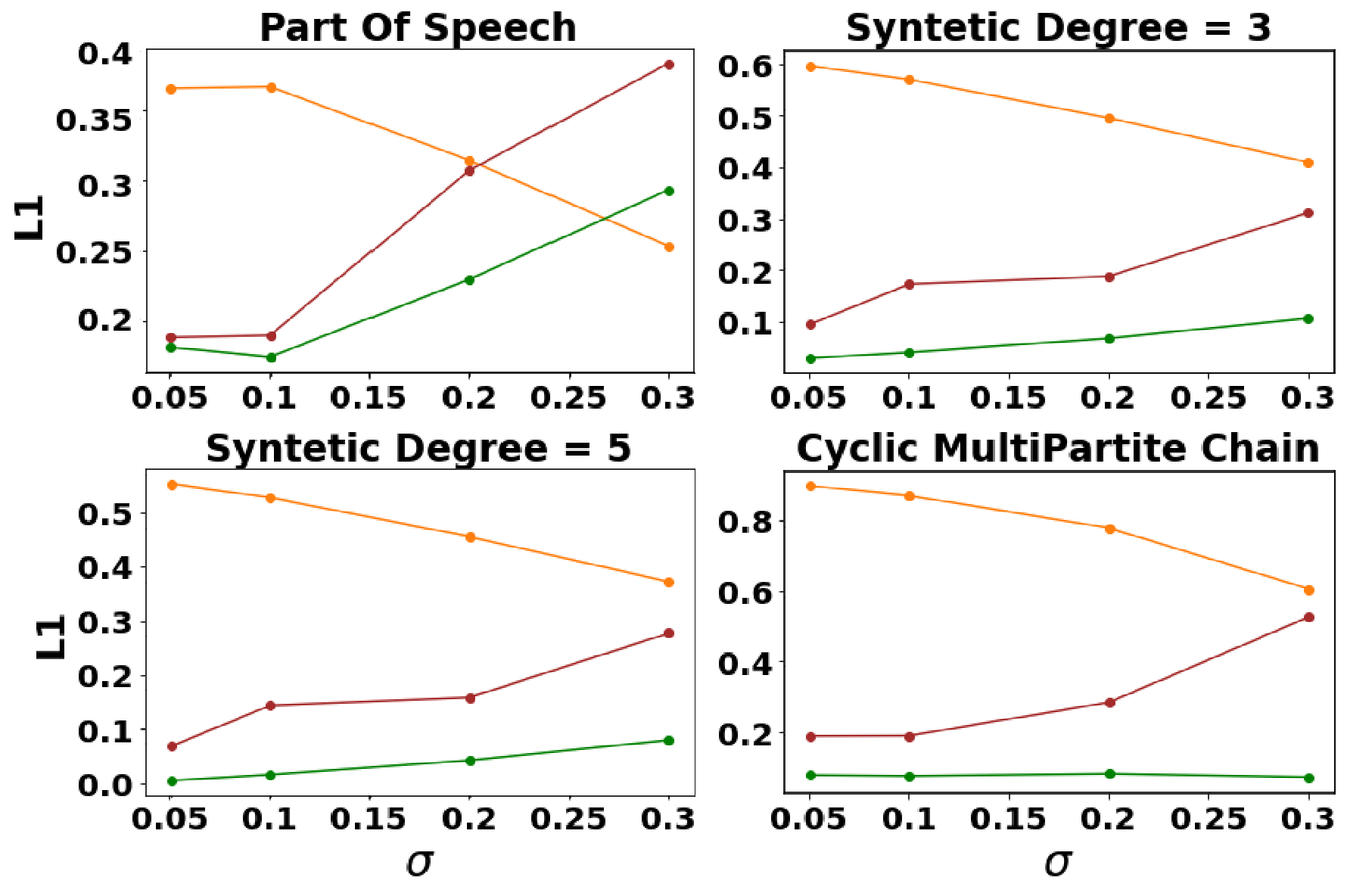}
    \caption{Non-constant $\mathbf{p_c}$.}
    \label{fig:robuNonConsPC}
\end{subfigure}
\hfill
\begin{subfigure}{0.32\textwidth}
       \includegraphics[width=\textwidth]{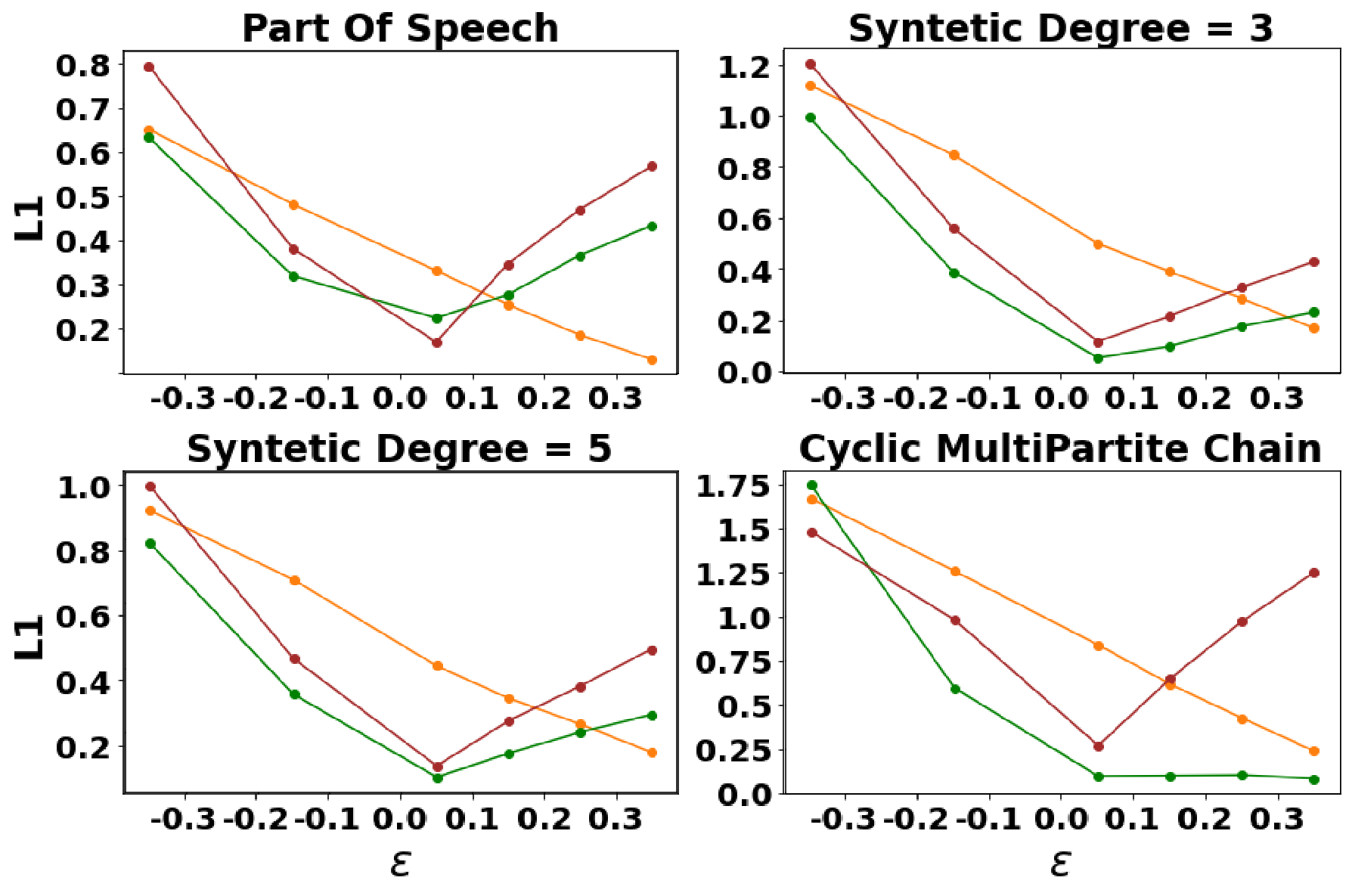}
       \caption{$\mathbf{\Phi_{p_c}}$ with memory.}
    \label{fig:robuNoniisPC}
\end{subfigure}
\hfill
\caption{Robustness Under Misspecification.}
\label{fig:missps}
\end{figure*}

From a practical point of view, it is reasonable to assume that $\Psi$ will be known to the reconstruction algorithms only up to some error. Therefore in this section we evaluate the robustness of the algorithms under different misspecifications.
In view of the results of the previous section, in this section we focus on the "Gaps sampler". Additionally, we presents the results of the semi-analytical approach as a comparison to the ignorable case.

Figure \ref{fig:wrongPC} presents the case of ``wrong $p_c$''. Here we address the ignorable case where the $p_c$ estimate provided to the algorithm differs from the ground truth value $p_c = \half$. As the figure shows, \textbf{both of the algorithms are relatively robust with respect to wrong $p_c$}. Moreover, note that even for sizable error in $p_c$ both algorithms are superior to the \emph{Naive} approach. In addition, the sampler based algorithm tends to be better then the semi analytical algorithm.

Figure \ref{fig:robuNonConsPC} presents the case of "non-constant" $p_c$, where there is no single constant ground truth $p_c$. Instead, for each sentence, $p_c$ is sampled independently from a normal distribution $p_c \sim N(\half,\sigma)$. In this case, the $p_c$ provided to the algorithm is the average one ($p_c = \half$). Notice that the larger the $\sigma$, the noisier the data is, and hence the harder the problem. Never the less, the Gaps sampler give great results even for high values of  $\sigma$, outperforming the semi-analytic method in a more pronounced way.
Given that the case of non constant $p_c$ is the most realistic one, this result showcases the \textbf{benefit in sampling based methods}.

\textbf{Robustness under Observations Removal Process Misspecifications.}\label{ssec:noiid}
Previous sections assume that the omission process,$\Phi_{\Psi}$ (Section \ref{ProblemSetup}), is modeled as a series of  independent Bernoulli trials. We now consider a case where instead the omission process is a Markov process with two states: $R_s="seen"$ and $R_m="missing"$, with transition probabilities given by 
$P(R_s|R_s) = p_c + \epsilon$ (and hence  $P(R_m|R_s) = p_c - \epsilon$), and  $P(R_s|R_m) = (1-p_c) - \epsilon$ (and hence $P(R_m|R_m) = (1-p_c) + \epsilon$). 
Thus, $\epsilon$ represents the bias of staying at the same state, with $\epsilon = 0$ case being equivalent to the i.i.d Brenoulli process above. 
Figure \ref{fig:robuNoniisPC} compares the algorithms for different values of the bias $\epsilon$. 
As the figure shows, the sampler results are mostly better than the ones of the \emph{Naive} solution, and constantly better then the semi-analytical ones. Nevertheless, the results of the \emph{Naive} solution improved as $\epsilon$ increases, and after some threshold of $\epsilon$, the \emph{Naive} algorithm become somewhat better then the sampler.
As discussed in Proposition \ref{MMM}, the \emph{Naive} reconstruction can be seen as the weighted mean of the ground-truth matrix T and wrong transition matrices $T^d$ derived from sequences of length d of omitted observations only. As $\epsilon$ increases, the probability of longer consecutive sequences increases (that is, long sequences of $P(R_s|R_s)$ or $P(R_m|R_m)$). But, while longer sequences of consecutive observed states ($R_s$ ) are beneficial to the $Naive$ solution, the solution is negatively effected by the number of omitted observations sequences, and not their length. 

\textbf{Robustness under Ignorability Assumption Misspecifications.}
Figure \ref{fig:nonignorableMiss} examines the impact of considering the non-ignorable setup on the performance of the methods.
For this, for each $s$ we sample $\Psi(s) \sim Uniform[0.5 - \epsilon , 0.5 + \epsilon]$ for varying levels of $\epsilon$ (X-axis), and provide the algorithms with only the empirical $p_c$, or $N$ for the Matching Sampler (dashed lines). We also compare the results to those of the \emph{``ideal benchmark,"} which assumes fully observed missing locations but does not take into account the non-ignorability assumption. As the figure shows, while the performance is similar for low values of $\epsilon$, it becomes more pronounced for higher variance for both the Gaps Sampler and the \emph{``ideal benchmark."} This demonstrates the importance of considering the non-ignorability assumption, which is often overlooked in HMM reconstruction methods.
\begin{figure}[h]
\begin{center}
       \includegraphics[width = 0.35\textwidth]{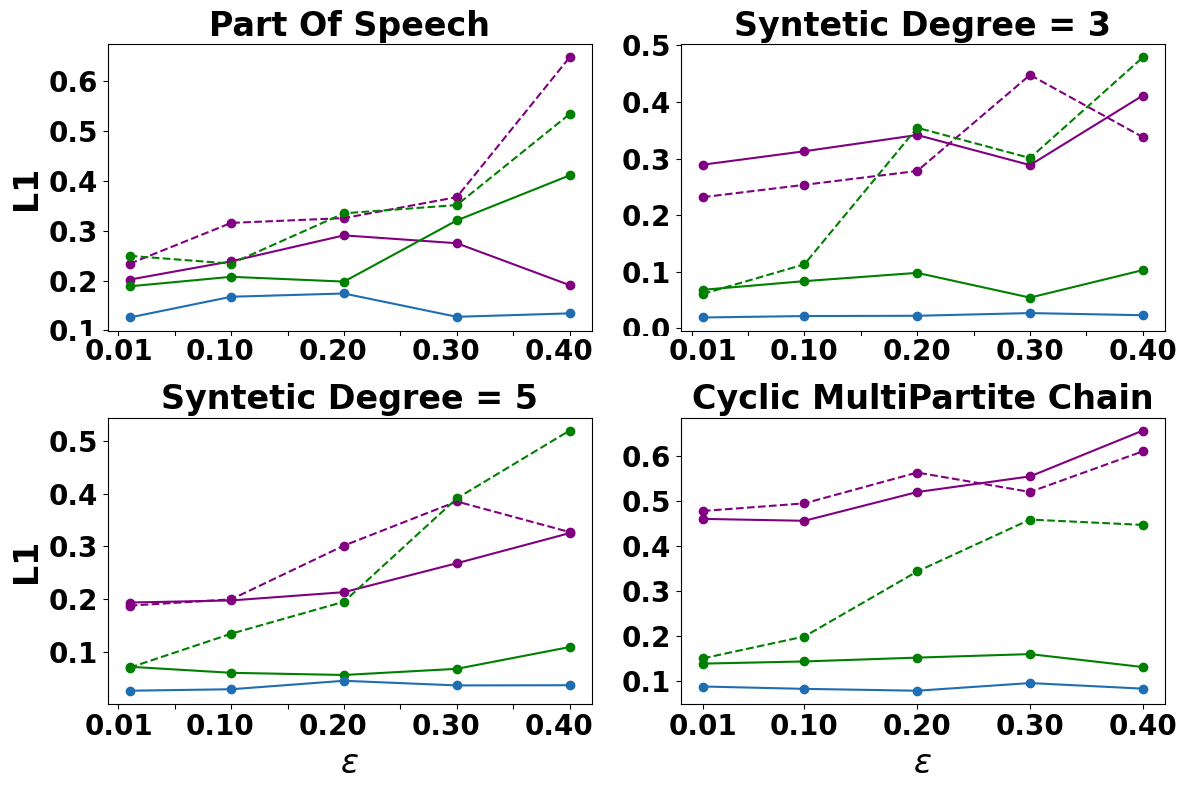}
       \vspace*{-5mm}
       \caption{Assume ignorable for non-ignorable $\Phi$. }
    \label{fig:nonignorableMiss}
\end{center}
\end{figure}

\subsection{Computational Complexity of the Proposed Algorithms}
The computational complexity of both the Gaps sampler and the Matching sampler depends on three factors: 
the complexity of sampling a single sentence, the number of of sentences, and the total number of iterations involved. When analyzing the complexity for a single sentence, the computational demands of both samplers are governed by the backward-forward sampling algorithm used for state sequence $X$ sampling, and additionally, by the sampling of either missing location ($W$) or interval length assignments ($D$). The backward-forward sampling algorithm has a well-understood complexity of $O(2 \cdot S^2 \cdot N)$, where $S$ symbolizes number of states and $N$ represents the length of the latent process. Also, for each time point $t$, a sampling procedure from a Dirichlet distribution of size $S$ is initiated.

Regarding to the sampling of interval lengths $D$, a forward algorithm is engaged (see Section \ref{sec:Gap}), carrying a complexity of $O(S^2 \cdot N)$, complemented by a singular sampling operation from a Dirichlet distribution of size 100, as detailed in Section \ref{sec:Gap}. The process of sampling $W$ presents the most computation-intensive part of this paper, utilizing an additional inner sampler. Its complexity hinges on the product of the count of pre-determined sampling iterations and $N$. 
As depicted in Figure \ref{fig:convergenceSpeed}, the Gaps parameters achieve faster convergence, which consequently leads to quicker Gibbs convergence.

As a specific example, the most computationally demanding experiment conducted in this paper (featuring 1000 sentences of size 100, a matching sampler with 120 iterations for $W$ learning, and the imputation of $\Psi$) would take an estimated 8 minutes. The majority of configurations, however, would require significantly less time. In our observations, this duration is considerably shorter than what the standard EM algorithm requires in the widely used Pomegranate package \citep{Schreiber2016}. The improvement in our computational performance is attributable to multiple elements, including efficient distribution sampling (especially Gaussian and Dirichlet) and full parallelism.

\section{Conclusion }
\label{sec:discussion}
This study addresses the challenge of reconstructing hidden Markov models from ordered trajectories with missing observations at unknown locations, which is a common issue in fields such as computational biology. A novel and general approach was proposed, based on Gibbs sampler, which overcomes many of the limitations of existing methods and opens the way for new applications. Additionally, the approach addresses the non-ignorable missing observations setting . Moreover, the robustness of the algorithms to different misspecifications was demonstrated. Two notable results were presented: 1) the reconstruction performances are comparable to the performance of algorithms that use known missing locations, even in the non-ignorable case, and 2) partially known omission probabilities can be inferred from data. As future work, further improvement is possible by incorporating prior knowledge of the latent process structure into the Bayesian framework and by using specialized omitting process tailored to specific problems. 
 
\section*{Acknowledgements}
This project has received funding from  the European Union’s Horizon Europe Programme under grant agreement No.101084642

\bibliography{HMMOP}
\bibliographystyle{icml2023}

%%%%%%%%%%%%%%%%%%%%%%%%%%%%%%%%%%%%%%%%%%%%%%%%%%%%%%%%%%%%%%%%%%%%%%%%%%%%%%%
%%%%%%%%%%%%%%%%%%%%%%%%%%%%%%%%%%%%%%%%%%%%%%%%%%%%%%%%%%%%%%%%%%%%%%%%%%%%%%%
% APPENDIX
%%%%%%%%%%%%%%%%%%%%%%%%%%%%%%%%%%%%%%%%%%%%%%%%%%%%%%%%%%%%%%%%%%%%%%%%%%%%%%%
%%%%%%%%%%%%%%%%%%%%%%%%%%%%%%%%%%%%%%%%%%%%%%%%%%%%%%%%%%%%%%%%%%%%%%%%%%%%%%%
\newpage
\appendix
\onecolumn
%%%%%%%%%%%%%%%%%%%%%%%%%%%%%%%%%%%%%%%%%%%%%%%%%%%%%%%%%%%%
\section{Naive Reconstruction and Sensitivity w.r.t $W$ Placement}\label{Supp:wrognW}
Figure \ref{Suppfig:Intro} presents the performance of the \emph{Naive} method for additional models.\\
Figure \ref{fig:wrongW} demonstrate the effects of gap location under different perturbation scenarios.
\emph{The first} perturbation (named ``Equivalent") assumes the new locations points to the same states in the original sentence, For example, for sentence [A,B,B,C] the locations vector (0,1,3) and (0,2,3) are equivalent. \emph{The second} permutation (named ``Consecutive") assumes that the consecutive observations(observations without gaps between them) locations are preserved and the number of consecutive observations is the same. For example, if the non-gaps locations are (1,2,5,7,8,10) we create (1,2,4,7,8,16). As the experiments shows, ignoring the missing observations or randomizing them gives quite bad results, although, the reconstruction not necessarily relays on the exact gaps as shown in the ``equivalent`` case.    

\begin{figure}[h]

\begin{center}
        \includegraphics[scale=0.2]{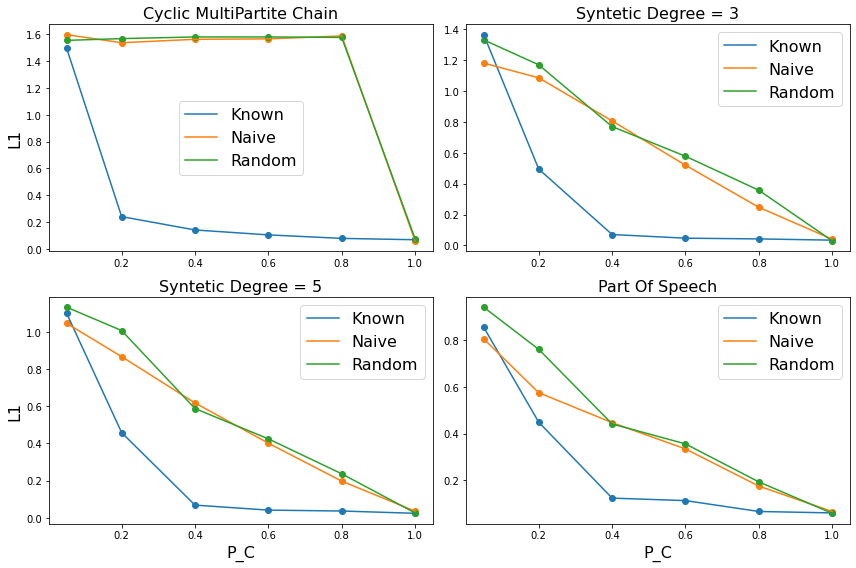}
      \caption{Reconstruction with the \emph{Naive} model and with random allocations of missing observations locations.
    \label{Suppfig:Intro}}
\end{center}
\end{figure}

\begin{figure}[h]
\begin{center}
    \includegraphics[scale=0.3]{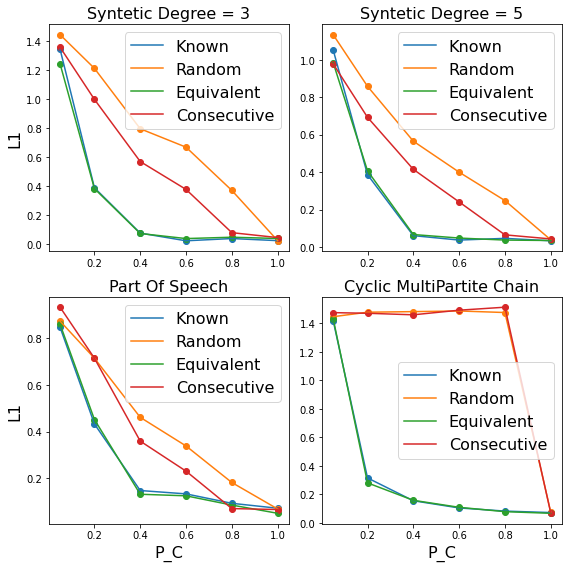}
    \caption{Reconstruction under different permutations of missing observation locations. X-axis is the $p_c$ of $Phi_{p_c}$, Y-axis is the L1 distance from the original T. 
    \label{fig:wrongW}}
\end{center}
\end{figure}

\section{Handling missing observations in HMM} \label{suppHMMKNOWN}
Missing observations are an integral part in the practical use of HMMs. This section presents a general approach of handling missing observation in the case of known locations.

Lets define a simple discrete Markov model 
where $O_{1:N} = (O_1, . . . , O_N )$ denote a time series of observations, and let $\Theta$ denote a vector of parameters. A hidden Markov model associates observations with a time series of hidden (or latent) discrete states $X_{1:N} = (X_1, . . . , X_N )$. $T_{X_t,X_{t - 1}}= p(X_t|X_{t−1})$ is the transitions probabilities between states. The joint distribution of observations and
states can be stated as
\begin{equation}
    P(O_{0:T},X_{0:T} ,\Theta) = P(X_0|\Theta)P(O_0|X_0,\Theta) \prod_{t=1}^{T} T_{X_t,X_{t-1}} \cdot P(O_t|X_t,\Theta)
\end{equation}
Now let us assume some of the observations $O_{1:T}$ are missing.  $O^{miss}$ will be the set of all missing observations and $O^{obs}$ the set for all observed observations. We can redefine as follow : \\

\begin{equation}
P(O_t|X_t,\Theta) = \mathbb{I}_{O_t \in O^{obs}}P(O_t|X_t,\Theta) + \mathbb{I}_{O_t \in O^{miss}} P(O_t = \emptyset|X_t,\Theta)    
\end{equation}

In our case the process removing the observations is dependent of  $\Theta$,$X$ or $O$ (non-ignorable) : 

\begin{equation}
\begin{split}
    = \mathbb{I}_{O_t \in O^{obs}} P(O_t|X_t,\Theta) \cdot (1-\Psi(X_t)) + \mathbb{I}_{O_t \in O^{miss}} \Psi(X_t)
\end{split}
\end{equation}

Hence, if we define $W[w_1,w_2,...]$ as a mapping between the $O_t \in O^{obs}$ to the corresponding time t (as in man paper), we can write: 
\begin{equation}\label{eq:eq5}
    \begin{split}
         P(X, T, \Theta|W, \mathcal{O},\Psi)& = \\  P(X_0|\Theta)P(O_0|X_0,\Theta)\cdot  &\prod_{t=1}^{N} T_{X_t,X_{t-1}} \cdot \begin{cases}
    (1-\Psi(X_t)) \cdot P(O_t|X_t,\Theta) , & \text{if}\ t \in W_i \\
  \Psi(X_t), & \text{otherwise}
\end{cases}
    \end{split}
\end{equation}

\section{Matching Sampler}\label{GibbsSamplerKnownN}
This section contains a fully detailed Gibbs sampler for the case of reconstructing HMMOPs with known N. The section begin with addressing the initialization of the sampler for HMMOPs reconstruction, proceed to describe the "known location" Gibbs sampler while assuming W is known, and finally described the process of sampling W with more details.

Since the precise nature of emissions is not critical for the general results, from now on we will concentrate on the case of HMM with scalar Gaussian emissions, $\Theta_{x_n}= N(\mu_n,\sigma)$, where $\mu_n$ is a learnable parameter.

We starts with presenting the parameters of our model :
\begin{itemize}
    \item $T$ is the matrix which holds the probabilities to move from state $X_i$ to $X_j$, with a Dirichlet distribution prior.
    \item $ \mathcal{X}$ is the set of all latent states sequence $X_i$.
    \item $\mathcal{W}$ is the set of all mapping $W_i$ corresponding to $X_i$, $O_i$.
\end{itemize}
The external parameters of the model are : 
\begin{itemize}
    \item $\Set{N_i}$  the expected length of each $|U_i| $ (or $|X_i|$ equivalently) . 
    \item $\sigma_j$ the standard deviation(S.D) of the distributions of X's emissions.
    \item $\{O_i\}$ Our input (Missing sentences). 
    \item $\Psi$ are the omitting probabilities.
\end{itemize}

For simplicity, given that we sample each $W_i$ or $X_i$ independently, this section proceed with the notations -  $O = [o_0,\ldots, o_K] := O_i$, $W = [w_0,\ldots, w_K] := W_i$ , $X = [x_0,\ldots, x_N] := X_i$. Also, we define $W^k$ as the sequence $W$ when the $w_k$ is omitted, and $X^{[a,b]}$, $O^{[a,b]}$ as the sub-sequence of $X$, $O$ starting with index $a$ and end with $b$.

\subsubsection{Calculating Initial Values}\label{CalculatingInitialValues}
\cite{10.1214/08-BA326} offers the following steps for drawing initial conditions for $\mu$ given Normal prior and where $\hat{O_i}$ is the "full observed trajectory" ($U_i$, not exist in our case)-
Each $\mu_i$ is given an independent Normal prior $N(\xi, \kappa^{−1})$ with $\xi = \frac{min_t(\hat{O}_i,_t) + max_t(\hat{O}_i,_t)}{2}$
and $\kappa = \frac{1}{R^2}, R = max_t(\hat{O}_i,_t) - min_t(\hat{O}_i,_t)$ . 
Because we cant map observations to times without W, we cant directly draw the initial conditions as in the simple case. So, we used an unsupervised approach to asses the mapping between observations and distributions. Given the prior of Normally distributed states, our observations are a set of samples drawn from the GMM $\sum_l^{|D|} d_l$, so for assigning observations to distributions we used an EM algorithm for GMM .
Then, for $\{O_l\}$ all the observations assigned to distribution $d_l$, we calculate $\xi_l = \frac{min(O_k) + max(O_k)}{2} $
and $\kappa_l = \frac{1}{R^2}, R = max(O_k) - min_t(O_k)$. \\
Initial conditions for $T$ and $W$ are drawn in random .

\subsubsection{Sampling Conditional Distributions for known W }
Gibbs sampler is a common algorithm for HMM reconstruction, which allow us to draw samples from the posterior distribution of the HMMs (or in our case \ref{eq:eq5}). Here we will follow the work of \cite{10.1214/08-BA326}.
The idea beyond Gibbs sampling is alternating between
sampling model parameters and latent data from their respective full conditional distributions.
This is because, given the latent Markov chain and the data, the parameters are
conditionally independent. And vice versa, given the parameters and the data the latent process is a non-homogeneous Markov chain and hence simple to sample.
The basic steps for sampling the posterior $P(T, X, \Theta, W|\mathcal{O}, N, \Psi)$ for the reconstruction are 1. start with initial values for the parameters 2.sample a latent sequence X form the posterior given the parameters 3. sample the parameters form their conditional distribution, X and the observations. 
\\

\textit{\textbf{Sampling T.}}
\\
As we defined earlier, $T$ the transitions probabilities are assumed to be drawn from an $N$X$N$ Dirichlet distribution, so for sampling T we sample the distribution : 
\begin{equation}
 T|X,W,\mu = T|X \sim Dir(\sum_i n_{i1} + 1,  \sum_i n_{i2}  + 1, . . . , \sum_i n_{id}  + 1)   
\end{equation}
where $n_{ij}$ is the number of transitions from state i to state j over all X .

\textit{\textbf{Sampling $\mu$.}}
\\
Notice that for sampling $\mu$, T is no longer needed. $\mu_x$, the mean of the emissions distributions of state x, deepens on the observations correspond to X, hence, depends on both X and W. From \cite{10.1214/08-BA326} we know : 
\begin{equation}
\mu|... \sim N(\hat{\mu},\hat{\sigma}) ;
    \hat{\mu} = (S_i+\kappa \xi)/(n_i+\kappa) ;
    \hat{\sigma} = 1/(n_i+\kappa)
\end{equation}
Where: 
$$S_n = \sum_{i,k} O_{ik} \cdot \delta_{X_{iW_k} , d_n} $$ \footnote{$\delta_{xy}$ is the Kronecker Delta} 

\textit{\textbf{Sampling $X$.}}
\\
Given $\Theta$,T and $O_i$, $X|\ldots$ can be modeled as a (non-homogeneous) Markov chain with initial distribution:  
\begin{equation}
 P(X_0| . . .) \propto \widehat{N}(O_{0}| X_0 ) \cdot P(O^{[2,n]}|X_0) \ . \    \widehat{N}(O_{it| X_j )} =
\begin{cases}
(1-\Psi(X_t)) \cdot N(O_{it}| X_j ) &if \ t \in W_i \\ 
\Psi(X_j) &else
\end{cases}
\end{equation}
and transition probabilities:
$$P(X_k| X_{k - 1}) \propto T_{X_k,X_{k-1}} \cdot \widehat{N}(O_{k}| X_k) \cdot P(O^{[k+1,N]}|X_k) $$

For sampling X, we use an approach similar to the "forward backward" algorithm, called the "backward recursion forward sampling" algorithm \cite{CHIB199679}. The algorithm sample X from the Markov chain similarly to how F-B algorithm sample observations from an HMM. 

Finally, Our sampling algorithm is presented in algorithm \ref{algo:GibbsWithN}

\begin{algorithm}[tb]
	\caption{Gibbs sampler given N}
	\label{algo:GibbsWithN}
	\begin{algorithmic}
		\STATE infer $\xi $ and $\kappa$ from $\mathcal{O}$ and calculate $\mu$
		\STATE Build random T and W
        \STATE sample $X|\mathcal{O},T,W,\mu$.
	    \REPEAT
	    \STATE sample $\mu|X,\mathcal{O},W,\kappa,\xi,\sigma$.
	    \STATE sample $T|X$.
	    \STATE sample $W|X,\mathcal{O},\mu,\sigma,\Psi$.
        \STATE sample $X|T,W,\mathcal{O},\mu,\Psi$.
	    \UNTIL{convergence}
	\end{algorithmic} 
\end{algorithm}

\subsubsection{Sampling W}
Given the independence between the conditional distributions of each of the parameters in the Gibbs sampler, the difference between sampling $P(\mathcal{X},T,\Theta,\mathcal{W}|\mathcal{O},\Psi)$ instead of $P(\mathcal{X},T,\Theta|\mathcal{W},\mathcal{O},\Psi)$ lay in sampling $P(\mathcal{W}|T,\mu,\mathcal{X},\mathcal{O},\Psi)$. 

As described in Section \ref{ProblemSetup}, $\mathcal{W}$ are independent. Also, given $X_i$, T is no longer needed for the conditional distribution of $W$. So, our problem become sampling:
\begin{equation} \label{PW_1}
  P(W = [w_0,\dots, w_K]|\Theta,X,O,\Psi)  
\end{equation}

For sampling from \eqref{PW_1} we used a Gibbs sampler once again, that is, instead of sampling $W$ from \eqref{PW_1}, we iteratively sample $w_k$ from : 
\begin{equation} \label{supp:PWk_1}
  P(w_k|W^k,\Theta,X,O,\Psi)  
\end{equation}

Given $w_{k-1}<w_k<w_{k+1}$, the probability \eqref{supp:PWk_1} is of mapping a single observations $O_{k}$ to one of $X^{[w_{k-1},w_{k+1}]}$. And given the Markov property of X and that $Phi(\cdot)$ is a memoryless process, this mapping is independent from $w_{k'}, X_{k'} , O_{k'}$ for $k' \not\in \{k-1,k,k+1\}$. So \eqref{supp:PWk_1} become :
$$ P \left({w_k} \middle| O_{k},X^{[w_{k-1},w_{k+1}]},\Psi \right) \cdot  I_{[w_{k-1},w_{k+1}]}(w_k) \footnote{$I_{[w_{k-1},w_{k+1}]}(w_k)$ is the indicator function}$$

\begin{equation*} 
\begin{split}
 = \frac{P(O_{k}|X^{[w_{k-1},w_{k+1}]},w_k)}{P(O_{k}|X^{[w_{k-1},w_{k+1}]})}  \cdot P({w_k}|X^{[w_{k-1},w_{k+1}]})
 \cdot I_{[w_{k-1},w_{k+1}]}(w_k) \\
 = \frac{P(O_{k}|  X_{w_k})}{P(O_{k}|X^{[w_{k-1},w_{k+1}]})}  \cdot \left( \frac{(1-\Psi(X_{w_k}))}{\sum_{j}^{[w_{k-1},w_{k+1}]} (1-\Psi(X_{j}))}\right)
 \cdot I_{[w_{k-1},w_{k+1}]}(w_k)
\end{split}
\end{equation*}

and the denominator is independent of $w_k$ :
\begin{equation}
    \propto P(O_{k}|  X_{w_k}) \cdot (1-\Psi(X_{w_k})) 
  \cdot I_{[w_{k-1},w_{k+1}]}(w_k)  
\end{equation}

\subsubsection{M-H Algorithm } \label{MHAlgo}
For better convergence we used a special case of Gibbs sampling called the M-H algorithm \cite{zbMATH03349185}. M-H is a special case of Gibbs sampling where the update from the new iteration $W^{t+1}$ is conditioned with an \emph{acceptance ratio } $\alpha = \frac{P(W^{t+1}|X_i,O_i)}{P(W_{t}|X_i,O_i)}$. 
Lets start with describing $P(W|O_i,X_i)$  : 
\begin{equation}
P(W|{O_i},{X_i}) = \frac{{P({O_i}|W,{X_i}) \cdot P(W|{X_i})}}{{P({O_i}|{X_i})}}
\end{equation}
\textbf{For the ignorable case} where $\Psi$ is independent of X we have : 
\begin{equation}
= \frac{{P({O_i}|W,{X_i}) \cdot p_c^{|O|} \cdot (1-p_c)^{(|X|-|O|)} }}{{P({O_i}|{X_i})}}
\end{equation}
And because $P({O_i}|{X_i})$ is independent of W : 
\begin{equation}
\propto P(O_i|W,X_i) \cdot p_c^{|O|} \cdot (1-p_c)^{(|X|-|O|)}
\end{equation}

Given $P(O_i|W,X_i)$ is simply the probability of observation sequence $O_i$ on state sequence $X_i$ with known mapping, we can say:   
\begin{equation}
   \alpha = \frac{p_c^{|O|} \cdot (1-p_c)^{(|X|-|O|)} \cdot \prod\limits_{k = 0}^{K - 1} {P({o_k}|{X_{i{w^{t+1}_k}}})}}{p_c^{|O|} \cdot (1-p_c)^{(|X|-|O|)} \cdot \prod\limits_{k = 0}^{K - 1} {P({o_k}|{X_{i{w^t_k}}})}}  = \frac{\prod\limits_{k = 0}^{K - 1} {P({o_k}|{X_{i{w^{t+1}_k}}})}}{\prod\limits_{k = 0}^{K - 1} {P({o_k}|{X_{i{w^t_k}}})}} 
\end{equation}

\textbf{For the non-ignorable case} : 
Because $P({O_i}|{X_i})$ is independent of W, and given $P(O_i|W,X_i)$ is simply the probability of observation sequence $O_i$ on state sequence $X_i$ with known mapping : 

\begin{equation}
   \alpha = \frac{P(W^{t+1}|{X_i}) \cdot \prod\limits_{k = 0}^{K - 1} {P({o_k}|{X_{i{w^{t+1}_k}}})}}{ P(W^t|{X_i}) \cdot \prod\limits_{k = 0}^{K - 1} {P({o_k}|{X_{i{w^t_k}}})}}  
\end{equation}

While $P(W^{t+1}|{X_i}), P(W^{t+1}|{X_i})$ are hard to evaluate, the difference between them is only one mapping $w$. Lets assume $w_j$ is the mapping which been updated in time $t$ : 
\begin{equation}
   \alpha = \frac{(1-\Psi({X_{w^{t+1}_j}})) \cdot \prod\limits_{k = 0}^{K - 1} {P({o_k}|{X_{i{w^{t+1}_k}}})}}{(1- \Psi({X_{w^{t}_j}})) \cdot \prod\limits_{k = 0}^{K - 1} {P({o_k}|{X_{i{w^t_k}}})}}  
\end{equation}

The full sampling algorithm is described in Algorithm \ref{alg2}. 

\begin{algorithm}[ht] 
	\caption{M-H sampler For W}
	\label{alg2}
	\begin{algorithmic}
		\STATE Initial $W^0$ randomly 
	    \FOR{i=0,1...,Number of iterations}
	    \STATE $W^{t+1} = W^t$
	    \FOR { k = 0,1,...,K}
	    \STATE $W_k^{t+1} \sim P(w_k|w_{k-1}^{t+1},w_{k+1}^t)$
	    \ENDFOR
	    \STATE calculate $\alpha = \frac{P(W^{t+1})}{P(W^{t})}$
	    \STATE sample u from uniform distribution over [0,1]
	    \IF{u $\leq \alpha$}
	    \STATE $W^{t}=W^{t+1}$
	    \ELSE 
	    \STATE $W^{t+1} = W^t$
	    \ENDIF
	    \STATE $W^{t} = W^{t+1}$
	    \ENDFOR
	   %  \STATE W = $W^t$
	\end{algorithmic} 
	
\end{algorithm} 

\section{Gap sampler}\label{supp:samplingd}

Lets $M = (\mathbb{X},T,\Theta)$ be an HMMOP, and $M_d = (\mathbb{X},T^d,\Theta)$ be an HMM where $T^d$ is the d-step transition matrix of transition matrix $T$. $D_{\Psi}(X_i,X_{i+1})$ is a random variable distributed according to the distances between two observable states $X_i,X_{i+1}$. 
From now on, we define $d_i = [d_{i0},\ldots,d_{iK}] $ as the sequence of intervals. Note that $d_{ik}$ is corresponding to $ w_{k+1} - w_k$ in the previous notation, as the number of gaps between two observed states in the original full states sequence $U_i$. $\mathcal{D}$ is the set of all $d_i$. As before, we omit the index i given the formulation is the same between sequences.

When conditioning over d, the posterior distribution of the HMMOP can be written as :    
\begin{equation}
    P(X,T,\Theta |d, O, \Psi) = 
    P(X_{0}|\Theta,d_0) \cdot  P(O_{0}|X_{0},\Theta)
     \cdot  \prod_k^{K_i} P(X_{{k+1}}|X_{k},d_k) 
     \cdot P(O_{{k+1}}|X_{{k+1}},\Theta) 
\end{equation}

\textbf{\textit{Samplind d :}}
Our goal is to calculate : 
$$P(d_k |(X_k,X_{k+1})) $$
Lets define $s = [s_0,\ldots,s_{d-1}] \in S^d$ as the set of all $d$ long states trajectories. Also lets assume d is limited to be no longer then a predefined number $\mathcal{S}$. We say that a sequence is \emph{fully omitted} if we didn't observed any state from the sequence. We denote the case of a fully omitted s as $\Bar{s}$. The probability of $s$ to be $\Bar{s}$ :
$$ P(\Bar{s}|s) = \prod_{i=0}^{d-1} \Psi(s_i) $$

By definition, $P(d_k |(X_k,X_{k+1}))$ is the probability for a sequence of length $d$ (hence $s \in S^d$), been fully omitted, between $(X_k,X_{k+1})$.
\begin{equation}
\begin{split}
 \frac{\sum_{s \in S^d} 
   P(s|(X_k,X_{k+1})) \cdot P(\Bar{s}|s)}{\sum_{\tau = 0}^{\mathcal{S}} \sum_{s' \in S^{\tau}} P(s'|(X_k,X_{k+1}))\cdot P(\Bar{s'}|s')}  = \\
 \frac{\sum_{s \in S^d} T[X_k,s_0] \cdot T[s_{d-1}, X_{k+1}] \cdot \Psi(s_0) \cdot \Pi_{i=1}^{d-1} T[s_{i-1}, s_{i}] \cdot \Psi(s_i)
  }{\sum_{\tau = 0}^{\mathcal{S}} \sum_{s' \in S^{\tau}} P(s'|(X_k,X_{k+1}))\cdot P(\Bar{s'}|s')}
\end{split}
\end{equation}

Given the limited number of d's, we only need the probability up to proportion : 
\begin{equation}\label{p_d}
    \propto \sum_{s \in S^d} 
  T[X_k,s_0] \cdot T[s_{d-1}, X_{k+1}] \cdot \Psi(s_0) \cdot \Pi_{i=1}^{d-1} T[s_{i-1}, s_{i}] \cdot \Psi(s_i)
\end{equation}

Lets present the case of $\Psi(\cdot) = (1-p_c)$ : 
\begin{equation}
\begin{split}
    = \sum_{s \in S^d} 
  T[X_k,s_0] \cdot T[s_{d-1}, X_{k+1}] \cdot (1-p_c) \cdot \Pi_{i=1}^{d-1} T[s_{i-1}, s_{i}] \cdot (1-p_c) \\
  =  (1-p_c)^{d-1} \cdot \sum_{s \in S^d} 
  T[X_k,s_0] \cdot \Pi_{i=1}^{d-1} T[s_{i-1}, s_{i}] \cdot T[s_{d-1}, X_{k+1}]  \\
  =(1-p_c)^{d-1} \cdot T^{d}[X_k, X_{k+1}] 
\end{split}
\end{equation}

Back to the general case, we evaluate \ref{p_d} using the forward algorithm. We can say that : 
$$P(d|(X_k,X_{k+1})) = \sum_{x \in \mathcal{X}} \sum_{j \in \mathcal{X}} P(d-1|(X_k,X_{k+1}),s_{d-2} = j) \cdot  T[j,x] \cdot (1- \Psi(x))$$
While :
$$P(d-1|(X_k,X_{k+1})) = \sum_{j \in \mathcal{X}} P(d-1|(X_k,X_{k+1}),s_{d-2} = j)$$

As we can see this is the form of the forward algorithm with $P(d|(X_k,X_{k+1}),s_{d-1} = j)$ as the forwarding elements.
Notice that $P(1,\ldots, d-1)$ are intermediate steps for calculating $p(d)$ in the forward algorithm, hence the complexity of calculating $P(1), \ldots, P(\mathcal{S})$ is equal to the complexity of calculating $P(\mathcal{S})$.\\ 

\textbf{\textit{Sampling X :}}
Given the equivalence between $\{ d \}$ to $\{ w \}$, we sample $X$ in the same way as with the Matching sampler.

\textbf{\textit{Sampling T :}}
A challenge in the new representation is sampling $T|X,d, \Psi$. Notice that even in the ignorable case, X contains samples from $T^d$ rather than T, and because of the difficulty in finding roots for stochastic matrices \cite{HIGHAM2011448} and the sparsity of most of $T^d$, it is hard to calculate $T^d$ explicitly in order to sample T. For those reasons, we use the fact that $w_k= \sum_{j=0}^{j=k} d_j$ and $N = \sum d_k$ and proceed as with sampling $P(T|X,\Theta,W,O,N)$. 

Algorithm \ref{algo:reconsKnownPC} describes the full Gibbs sampler.

\begin{algorithm}[tb] 
	\caption{Gibbs sampler given PC}
	\label{algo:reconsKnownPC}
	\begin{algorithmic}
		\STATE infer $\xi $ and $\kappa$ from $\mathcal{O}$ and calculate $\mu$
	    \STATE sample T from a uniform Dirichlet prior 
        \STATE sample $X_{walk}|\mathcal{O},T,\mu$.(naive sampling given T )
        \STATE sample $d|T,X_{walk},p_c$
	    \REPEAT
	    \STATE sample $\mu|X_{walk},\mathcal{O},\kappa,\xi,\sigma$
	    \STATE build W,N from d
        \STATE sample $X_T|W,N,T,\mathcal{O},\mu ,\Psi$
        \STATE sample $T|X_T$
        \STATE sample $X_{walk}|T,d,\mathcal{O},\mu,\Psi$
        \STATE sample $d|T,X_{walk},\Psi$
	    \UNTIL{convergence}
		
	\end{algorithmic} 
\end{algorithm}

\section{Inference}\label{Supp:Inference}
A Common use of Hmm is sequence labeling (or inference, interchangeably). That is, given observations sequence $O$ and an HMM with known parameters we aim to find  $X_{ml} = argmax_X \ P(X|W,O)$. This section provide a inference method for sequence drawn from a HMMOP, so our aim is to find : 
\begin{equation}\label{Infer}
    X_{ml} = argmax_X \ P(X,W|O)
\end{equation}
Notice, that this case refer to a scenario where T from the fully observed HMM is known, but, the sentences have been drawn from the corresponding HMMOP. Given HMMOP is an HMM (\ref{MMM}), the best representation for the dynamics of the labels drawn from the HMMOP is $T_{r}$. So, this case reefer to a scenario where $T_r$ cannot be learn, for example because of lack of training data, but the "original" T is available. 

Many algorithms exists for the case of known W. So, for solving \eqref{Infer} we use an iterative Expectation-maximization (E-M) algorithm based on the sub-problem :
\begin{equation}
\begin{split}
    W_{ml} = argmax_W \ P(W|X,O)
\end{split}
\end{equation}
As before, $W=[W_0,\ldots,W_K]$ are mappings between $O_k \in O$ to $X_n \in X$. So, our goal is to find:
\begin{equation}\label{argW}
    \begin{split}
          argmax_W \ \prod_{k=0}^{K-1} P(O_k|X_{W_k}) =& \ argmax_W \sum_{k=0}^{K-1} log(P(O_k|X_{W_k})) \\ 
        W_{k-1} < W_{k}& < W_{k+1}
    \end{split}
\end{equation}
This problem can be represented as finding longest path on a directed acyclic graph (DAG). Lets define a DAG $G= (S,T)$ where $s_{a,b} \in S, a \in [0,K-1], b \in [0,N-1]$ are the nodes with $r_{a,b}$ as weights and $T'$ as the binary matrix representing the edges :
\begin{equation}\label{nodesum}
r_{a,b} = 
\begin{cases}
log(P(O_a|X_{b})) \ \text{ ,if }  b \geq k \\
-\infty \ \ else 
\end{cases}     
T[s_{a_i,b_i},s_{a_j,b_j}] = 
\begin{cases}
1 \ \text{ ,if } \ (a_j - a_i) == 1  \And b_j > b_i \\
-\infty \ \ else 
\end{cases}
\end{equation}

% directed graph where the k step on the graph represent an assignment for $W_k$. The directionality and acyclicity of the graph will insure the condition on W order ($W_{k-1} < W_{k} < W_{k+1}$), And the weight on the nodes will provide the optimization target. We define as follow : \\
Now lets consider the problem of finding the longest path $V' = [s'_{a_o,b_o},\ldots,s'_{a_K,b_k}]$ on $G$ that starts in one of $s_{0,:}$ and ends in one of $s_{:,N}$. Given $T'$, only transitions with consecutive $a$ are possible, so, $|V'| = K$. Also, $V'$ is ordered for $b$, so $V'$ is a sequence maximizing a sum of weights \eqref{nodesum}, where $a$ are consecutive and $b$ are ordered. In other words, by taking $b$ for each $s' \in V'$ we get the optimal W. Figure \ref{fig:illoptimal} present an illustration of the process. 

For better initialization, we calculated an initial $W_{ml}$ instead of using random allocation. Given T, $P(d_k=n|X_k,X_{k+1})$ is given by \ref{lemma_d}. Our goal is to solve : 
\begin{equation}
\begin{split}
argmax_d \prod_{k=0}^{K} P(d_k|X_k,X_{k+1}) &= argmax_d \sum_{k=0}^{K} log(P(d_k|X_k,X_{k+1})) \\ 
d_{K+1} &+ \sum d_k  < N         
\end{split}
\end{equation}

where $d_{K+1}$ is the probability for d gaps from $X_K$ to end. This problem is known as the multi-choice knapsack problem where one pick a single item $d_k$ from {1,...,N-K} per k, that maximizes a "benefit" ($\sum_k P(d_k)$) with constrain over some "cost" ($\sum d_k < N$).
The details of the multi-choice knapsack algorithm are in the code.

\begin{figure}
    \centering
    \includegraphics[scale=0.5]{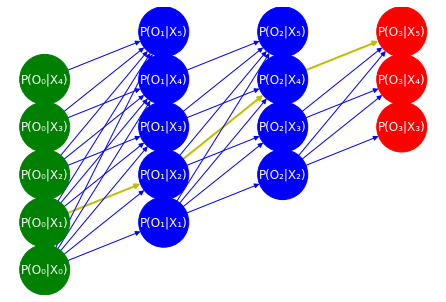}
    \caption{We want to find the longest path that start in green node and end in red node. In this case. Here [$log(P(O_0|X_1)),log(P(O_1|X_2)),log(P(O_2|X_4)),log(P(O_3|X_5))$]  so $W_{ml} = [1,2,4,5]$.} 
    \label{fig:illoptimal}
\end{figure}

\subsection{Inference results - Sequence Labeling}
As presented in \ref{Supp:Inference}, the reference algorithm receive 3 inputs - 1) a transition matrix T (derived from a fully observed HMM). 2) missing sentences 3) assumption for the percentage of missing observations, in our case the original sentence length N. Than, the algorithm returns a sequence of states of the length of the observations sequence. 

For this experiment we sampled sentences from the models in \ref{Supp:EvaluationsModel} and remove observation according to $p_c = .5$. Than, we predict the labels of each sentence and compare them to the known ones. The evaluation measures are the mean and variance of the accuracy across sentences. We compare the results of five algorithms: 1) naive prediction on the full sentences (named ``Full sentence"). 2) Based on the emissions probabilities only (``Emissions only"). 3) Missing sentences when the gaps locations are known (``Naive"). 4) missing sentences using a naive prediction. 5) missing sentences using the HMMOP inference with known N(``HMMOP"). 

Given the sensitivity of the label prediction task to the emissions probabilities, for the Gaussian models we compare the results for different $\sigma s$. As Table \ref{sample-table} present, the HMMOP give better results under all scenarios. Also, its interesting to see that when ignoring the missing observations, one better relay on the emissions alone rather using the wrong dynamical data. 

\begin{table}
  \caption{Label prediction results}
  \label{sample-table}
  \centering
  \begin{tabular}{lllllll}
    \cmidrule(r){2-7}
    \multicolumn{4}{c}{Synthetic Degree = 5 } & \multicolumn{3}{c}{Synthetic Degree = 3 }  \\
    \cmidrule(r){2-7}
    STD & 0.5 &  0.75 & 1.0 & 0.5 &  0.75 & 1.0 \\
    % \midrule 
    \cmidrule(r){2-7}
    Full sentence & .87(.08) & .80(.11)  & .64(.10) & .96(.05) & .98(.05)  & .82(5.6) \\
    Emissions only& .69(.15) & .57(.15)  & .42(.17) & .51(.15) & .72(.13)  & .43(5.6) \\
    Known W       & .79(.15) & .65(.16)  & .58(.17) & .90(.12) & .89(.10)  & .60(5.6)  \\
    Naive         & .66(.15) & .52(.17)  & .49(.18) & .80(.21) & .49(.18)  & .38(5.6)  \\
    HMMOP         & .73(.14) & .59(.14)  & .50(.17) & .84(.16) & .73(.15)  & .45(5.6)  \\
    \cmidrule(r){2-7}
  \end{tabular}
  \begin{tabular}{llllll}
    
    \multicolumn{3}{c}{Multi-Partite} & \multicolumn{3}{c}{POS}  \\
    \cmidrule(r){1-6}
     0.5     &  1.0     & 1.5  \\
    \midrule
    .91(.06) & .78(.09)  & .69(.12) & .91(.06)  \\
    .71(.15) & .50(.15)  & .40(.16) & .89(.11)  \\
    .83(.13) & .64(.17)  & .57(.18) & .90(.10)  \\
    .69(.16) & .48(.18)  & .40(.17) & .88(.11)  \\
    .72(.16) & .50(.16)  & .40(.15) & .89(.12)  \\
    \bottomrule
  \end{tabular}
\end{table}

\section{Analytical Reconstruction - Ignorable Case}\label{suppFB}
Given $OP(\cdot)$, the \emph{Naive} algorithm reconstruct $T_{missing}$ which can be described as follow: 
\begin{equation}\label{suppeq1}
\begin{split}
    T_{missing}(N) \propto p_c \cdot T_{ab} + p_c \cdot \sum_{n=2}^{N} (1-p_c)^{n-1} \cdot T^n_{ab} 
   \\ = p_c \cdot T_{ab} \cdot [\mathbb{I} + \sum_{n=2}^{N} (1-p_c)^{n-1} \cdot T^{n-1}_{ab}]
   \\ = p_c \cdot T_{ab} \cdot [\mathbb{I} + \sum_{n=1}^{N} [(1-p_c) \cdot T_{ab}]^{n}]
   \\ = p_c \cdot T_{ab} \cdot \sum_{n=0}^{N} [(1-p_c) \cdot T_{ab}]^{n}
\end{split}
\end{equation}
And after normalization 
\begin{equation}\label{suppeq2}
\begin{split}
    T_{missing}(N) &= [PC \cdot \sum_{n=0}^{N} [(1-p_c)]^{n}]^{-1 }  \cdot  p_c \cdot T_{ab} \cdot \sum_{n=0}^{N} [(1-p_c) \cdot T_{ab}]^{n} \\ &= [PC \cdot \frac{1-(1-p_c)^{N+1}}{p_c} ]^{-1 }  \cdot  p_c \cdot T_{ab} \cdot \sum_{n=0}^{N} [(1-p_c) \cdot T_{ab}]^{n} \\ &= \frac{p_c}{1-(1-p_c)^{N+1}}  \cdot T_{ab} \cdot \sum_{n=0}^{N} [(1-p_c) \cdot T_{ab}]^{n}
\end{split}
\end{equation}
Where $N$ is the number of gaps in the sentence.
The highest eigenvalue of the stochastic matrix $T_{ab}$ is 1, so for $p_c \in [0,1) \xrightarrow[]{} |p_c \cdot T_{ab}| < 1$  and we can write : 
\begin{equation}
    T_{missing}(N) = \frac{p_c}{1-(1-p_c)^{N+1}} \cdot T_{ab} \cdot [\mathbb{I} - (1-p_c) \cdot T_{ab}]^{-1} [\mathbb{I} - (1-p_c)^NT^N]
\end{equation}
\begin{equation}\label{suppeq8}
    T_{missing}(N) \xrightarrow[N]{\infty} p_c \cdot T_{ab} \cdot [\mathbb{I} - (1-p_c) \cdot T_{ab}]^{-1} 
\end{equation}
\\

As \ref{suppeq8} shows, given any T and $p_c$, we can infer $T_{missing}$ the transition matrix given $OP(\cdot)$. Notice that we define this transformation as the "backward transformation" earlier.\\
For the "forward transformation", which is the transformation from ($T_{missing}$,$p_c$) to T : 
\begin{multline*} \label{eq:eqbackward}
    T_{ab} = [\mathbb{I} \cdot p_c + (1-p_c)T_{missing}[\mathbb{I} - [(1-p_c) \cdot T_{ab}]^N]^{-1}]^{-1}\\  \cdot T_{missing} \cdot [\mathbb{I} - [(1-p_c) \cdot T_{ab}]^N]^{-1}
\end{multline*}
\begin{equation}\label{eq:eq10}
    \xrightarrow[N]{\infty} [\mathbb{I} \cdot p_c + (1-p_c) \cdot T_{missing}]^{-1} \cdot T_{missing}
\end{equation}
% Algorithm \ref{algo:SemiAna} present the semi analytical algorithm. 
From now on we will present the backward transformation as $T^{-p_c}$ and the forward transformation as $T^{p_c}$.

% \begin{algorithm}[tb]
% 	\caption{Semi Analytical algorithm}
% 	\label{algo:SemiAna}
% 	\begin{algorithmic}
% 		\STATE infer $\xi $ and $\kappa$ from $\mathcal{O}$ and calculate $\mu$
% 		\STATE Build random T and W
%         \STATE sample $X|\mathcal{O},T,W,\mu$.
% 	    \REPEAT
% 	    \STATE sample $\mu|X,\mathcal{O},W,\kappa,\xi,\sigma$.
% 	    \STATE sample $T|X$.
% 	    \STATE sample $W|X,\mathcal{O},\mu,\sigma$.
%         \STATE sample $X|T,W,\mathcal{O},\mu$.
% 	    \UNTIL{convergence}
% 	\end{algorithmic} 
% \end{algorithm} 

\begin{lemma} \label{Therom:statsDist}
$\pi$ the stationary distribution of $T_{ab}$ is equal to $\pi_m$ the stationary distribution of $T_{missing}$ for $N \xrightarrow{} \infty$
\end{lemma}
\begin{proof}
Given $\pi$ the stationary distribution of $T_{ab}$, and from \eqref{eq7}($K = [PC \cdot \sum_{n=0}^{N} (1-p_c)^{n}]^{-1 }$ normalization factor) :
\begin{equation}
\begin{split}
    \pi \cdot T_{missing} = K \cdot p_c \cdot \sum_{n=0}^{N} (1-p_c)^n \cdot \pi \cdot T_{ab}^{n+1} \\
    = K \cdot p_c \cdot \sum_{n=0}^{N} [(1-p_c)^n \cdot \pi] \\
    = \pi \cdot K \cdot p_c \cdot \sum_{n=0}^{N} (1-p_c)^n 
    = \pi 
\end{split}
\end{equation}
\end{proof}
In fact we can say that given $T_m$ a stochastic matrix build as a polynomial of stochastic matrix T, $T_m = \sum_{i} a_i \cdot T^i$ if $\pi$ is a stationary vector for T it is also a stationary vector for $T_m$. 

\subsubsection{Morphism}
\begin{lemma} \label{lemma:morphism}
$(T^{-Q_0})^{-Q}$  = $T^{-Q \cdot Q_0}$ and $(T^{Q_0})^{Q}$  = $T^{Q \cdot Q_0}$
\end{lemma}

\begin{proof}
\begin{equation}
    \begin{split}
        &(T^{-Q_0})^{-Q} = p_q \cdot T^{-Q_0} \cdot [\mathbb{I} - (1-p_q) \cdot T^{-Q_0}]^{-1} \\
        &=  p_q \cdot P_{Q_0} \cdot T_{ab} \cdot [\mathbb{I} - (1-P_{Q_0}) \cdot T_{ab}]^{-1} \cdot [\mathbb{I} - (1-p_q) \cdot P_{Q_0} \cdot T_{ab} \cdot [\mathbb{I} - (1-P_{Q_0}) \cdot T_{ab}]^{-1}]^{-1} 
        \\ &= p_q \cdot P_{Q_0} \cdot T_{ab} \cdot [[\mathbb{I} - (1-P_{Q_0}) \cdot T_{ab}] \cdot [\mathbb{I} - (1-p_q) \cdot P_{Q_0} \cdot T_{ab} \cdot [\mathbb{I} - (1-P_{Q_0}) \cdot T_{ab}]^{-1}]^{-1} 
        \\ &= p_q \cdot P_{Q_0} \cdot T_{ab} \cdot [[\mathbb{I} - (1-P_{Q_0}) \cdot T_{ab}] - (1-p_q) \cdot P_{Q_0} \cdot T_{ab}]^{-1}
        \\ &= (p_q \cdot P_{Q_0}) \cdot T_{ab} \cdot [[\mathbb{I} - (1-(p_q \cdot P_{Q_0})) \cdot T_{ab}]^{-1} = T^{-Q \cdot Q_0}
    \end{split}
\end{equation}
\\
\\
Following our resent prove : \\ \\ 
    $T^{-Q \cdot Q_0} = (T^{-Q_0})^{-Q}$ \\ 
    
We will use the forward transformation twice on both sides: \\ \\ 
$((T^{-Q \cdot Q_0})^{Q})^{Q_0} = (((T^{-Q_0})^{-Q})^{Q})^{Q_0}$ \\

More, we claim that $(T^{p_c})^{-p_c} = T$ because the forward transformation derived directly from the backward transformation so
$((T^{-Q \cdot Q_0})^{Q})^{Q_0} = T$ \\ 

Finally:
\begin{equation}
\begin{split}
    ((T^{-Q \cdot Q_0})^{Q})^{Q_0} &= T = (T^{-Q \cdot Q_0})^{Q \cdot Q_0} \\ 
    (T^Q)^{Q_0} &= (T)^{Q \cdot Q_0}
\end{split}
\end{equation}
\end{proof}

\subsubsection{Forward and backward composition}
We will start with calculating the composition of the forward to the backward transformations. This process can be described as trying to reconstruct T from $T_{missing}$ based on the wrong assumption of $\Phi_{\Psi}$ when the real process was $\Phi_{p_q}$  \\ \\
\begin{equation}
\begin{split}
    &(T^{-p_q})^{p_c} = [I \cdot p_c + (1-p_c) \cdot p_q \cdot T \cdot [I-(1-p_q) \cdot T]^{-1}]^{-1} \cdot p_q \cdot T \cdot [I-(1-p_q) \cdot T]^{-1} \\ 
    &((T^{-p_q})^{p_c})^{-1} = [I-(1-p_q) \cdot T] \cdot \frac{1}{p_q} \cdot T^{-1} \cdot [I \cdot p_c + (1-p_c) \cdot p_q \cdot T \cdot [I-(1-p_q) \cdot T]^{-1}] \\ 
    &= \frac{1}{p_q} \cdot [T^{-1}-(1-p_q) \cdot I] \cdot [I \cdot p_c + (1-p_c) \cdot p_q \cdot T \cdot [I-(1-p_q) \cdot T]^{-1}] \\
    &= \frac{1}{p_q} \cdot [T^{-1} \cdot p_c + (1-p_c) \cdot p_q \cdot T^{-1} \cdot T \cdot [I-(1-p_q) \cdot T]^{-1}\\
    & \ \ \ \ - I \cdot p_c \cdot (1-p_q) - (1-p_c) \cdot (1-p_q) \cdot p_q \cdot T \cdot [I-(1-p_q) \cdot T]^{-1}] \\
    &= \frac{p_c}{p_q} \cdot T^{-1} + [(1-p_c) - (1-p_c) \cdot (1-p_q) \cdot T] \cdot [I-(1-p_q)T]^{-1} - I \cdot (1-p_q) \cdot \frac{p_c}{p_q} \\
    &= \frac{p_c}{p_q} \cdot T^{-1} + (1-p_c)[I-(1-p_q)T] \cdot [I-(1-p_q)T]^{-1} - I \cdot (1-p_q) \cdot \frac{p_c}{p_q} \\
    &= \frac{p_c}{p_q} \cdot T^{-1} + I \cdot [(1-p_c) - (1-p_q) \cdot \frac{p_c}{p_q}] \\
    &= \frac{p_c}{p_q} \cdot T^{-1} + I \cdot \frac{p_q-p_c}{p_q} 
\end{split}
\end{equation}

\begin{equation}
\begin{split}
    &(T^{-p_q})^{p_c} = [\frac{p_c}{p_q} \cdot T^{-1} + I \cdot \frac{p_q-p_c}{p_q}]^{-1} \\&= [\frac{p_c}{p_q} \cdot T^{-1} + I \cdot (1-\frac{p_c}{p_q})]^{-1}
\end{split}
\end{equation}
\\
\\
Now we will use lemma \ref{lemma:morphism} to calculate $(T^{p_c})^{-p_q}$ by showing the forward and backward transformations are commutative : 
\begin{equation} \label{eq:FBComposition}
\begin{split}
    (T^{p_c})^{-p_q} = (((T^{-p_q})^{p_q})^{p_c})^{-p_q} &= \\
    (((T^{-p_q})^{p_c})^{p_q})^{-p_q} = (T^{-p_q})^{p_c} &= \\  [\frac{p_c}{p_q} \cdot T^{-1} + I \cdot (1-\frac{p_c}{p_q})]^{-1}&
\end{split}
\end{equation}
\\
\\
We can formulate a connection between T,$p_q$ and $T^{-p_q}$ from \eqref{eq:FBComposition}. Lets assume we observed $T^{-p_q}$ and someone gave us the transition matrix T, how can we infer $p_q$ ?  We can always use iterative method to solve this problem , but another option is to use the derivative of $\frac{d((T^{-p_q})^{p_c})^{-1}}{dp_c}$ which can be calculated regardless of the unknown $p_q$ ($p_c$,$p_c + \epsilon$ can be chose randomly):    \\ 

$\frac{d(((T^{-p_q})^{p_c})^{-1})}{dp_c} = \frac{1}{p_q} \cdot [T^{-1} - I]$. And :  \\
\begin{equation}
    \begin{split}
     p_q &= ({\frac{d(((T^{-p_q})^{p_c})^{-1})}{dp_c}})^{-1} \cdot [T^{-1} - I] \\
     T &= [p_q \cdot \frac{d(((T^{-p_q})^{p_c})^{-1})}{dp_c} + I]^{-1} 
    \end{split}
\end{equation}

Figure \ref{fig:stabilityBack} presents the sensitivity of the backward transformation to the $p_c$ used for the transformation. A Markov chain with 10 states was generated where the transition matrix $T_{r}$ was chosen at random(uniform). 1500 trajectories of length 100 were sampled proceeding by $\Phi_{\Psi}$ to build the empirical $T_{missing}$. Then, \eqref{eq:eqbackward} was used to reconstruct $T_{r}$.

\begin{figure}[ht]
\vskip 0.2in
\begin{center}
 \includegraphics[scale=0.2]{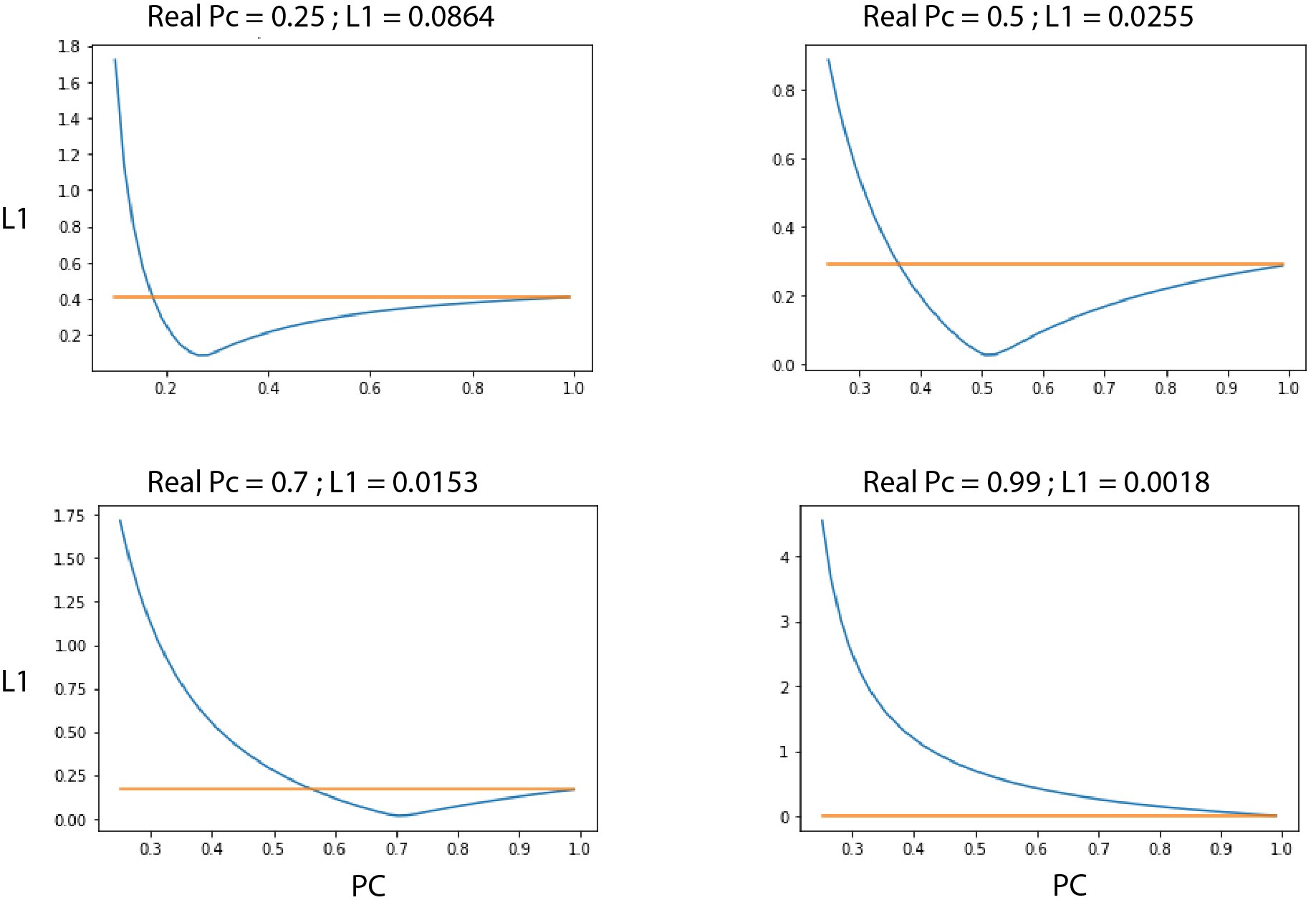}
    \caption{Backward transformation sensitivity to $p_c$. Each plot represents different ground truth $p_c$.Y-axis the L1 distance between T and $T_{r}$. X-axis the $p_c$ for the backward transformation. Blue is the transformation results Orange the ``naive" solution.  }
    \label{fig:stabilityBack}
\end{center}
\vskip -0.2in
\end{figure}

% \subsubsection{Transformations limits} 
% \begin{split* }
%     \lim_{p_c\to 1} T_{missing}(N) &= T_{ab} \cdot [\mathbb{I}]^{-1} [\mathbb{I}] = T_{ab} \\
%     \lim_{p_c\to 0} T_{missing}(N) &= \frac{p_c}{1-(1-p_c)^{N+1}} \cdot T_{ab} \cdot [\mathbb{I} - (1-p_c) \cdot T_{ab}]^{-1} [\mathbb{I} - (1-p_c)^NT^N] \\ 
%     &= T_{ab} \cdot [\mathbb{I}-T_{ab}]^{-1} \cdot [\mathbb{I}-T_{stationary}]
% \end{split*}
% \\
% \\
% \begin{split*}
%      \lim_{p_c\to 1} T_{ab} &= [\mathbb{I}]^{-1}  \cdot  T_{missing} \cdot [\mathbb{I} ]^{-1} = T_{missing}\\
%      \lim_{p_c\to 0} T_{ab} &= [T_{missing}[\mathbb{I} - [T_{ab}]^N]^{-1}]^{-1}  \cdot T_{missing} \cdot [\mathbb{I} - [T_{ab}]^N]^{-1} = I
% \end{split*}

\subsubsection{Robustness under iid misspecification}\label{SuppIidMisspecification}
Lets now describe the simplest algorithm  F = $Count(*)$ a function that count all the transitions in $\mathcal{O}$ and normalize to probability, as we know this is the maximum likelihood estimator for T with a Dirichlet prior. We define $T_{P} = Count(\Phi_{\Psi}(\mathcal{O}))$ and $T_{M0}$ as the transition matrix of the original Markov chain \\
As we showed before : \\

\begin{equation}\label{eq1}
\begin{split}
    T_{P} \propto p_c^2 \cdot T_{M_0} + p_c^2 \cdot \sum_{n=2}^{N} (1-p_c)^{n-1} \cdot T^n_{M_0} 
\end{split}
\end{equation}
And for ${\hat{\Phi}}$ : 
\begin{equation}
\begin{split}
    T_{\hat{P}} \propto P(R_s) \cdot P(R_s|R_s) \cdot T_{M_0} + P(R_s) \cdot P(R_s|R_s) \cdot P(R_s|R_s) \cdot \sum_{n=2}^{N} (P(R_s|R_s)^{n-2} \cdot T^n_{M_0} 
    \\= p_c \cdot (p_c + \epsilon) \cdot T_{M_0} + p_c \cdot (p_c - \epsilon) \cdot ((1-p_c) - \epsilon) \cdot \sum_{n=2}^{N} ((1-p_c) + \epsilon)^{n-2} \cdot T^n_{M_0} 
    \\ = T_{M_0}[p_c \cdot (p_c+\epsilon) \cdot I + p_c \cdot (p_c-\epsilon) \cdot \frac{1-p_c-\epsilon}{1-p_c + \epsilon} \cdot [\sum_{n=0}^N(1-p_c+\epsilon)^nT_{M0}^n - I]]
    \\ \xrightarrow[N]{\inf} T_{M_0} \cdot p_c \cdot (p_c+\epsilon) \cdot [I + \frac{(p_c-\epsilon)}{(p_c+\epsilon)} \cdot \frac{1-p_c-\epsilon}{1-p_c + \epsilon} \cdot [[I-(1-p_c+\epsilon) \cdot T_{M0}]^{-1} - I]]
\end{split}
\end{equation}
While the normalize expression is
\begin{equation}
\begin{split}
    T_{M_0} \cdot (p_c+\epsilon) \cdot [I + \frac{(p_c-\epsilon)}{(p_c+\epsilon)} \cdot \frac{1-p_c-\epsilon}{1-p_c +\epsilon} \cdot [[I-(1-p_c+\epsilon) \cdot T_{M0}]^{-1} - I]]\\
T_{M_0} \cdot (p_c+\epsilon) \cdot [I + \frac{(p_c-p_c^2+\epsilon^2) - \epsilon}{(p_c-p_c^2+\epsilon^2) + \epsilon} \cdot [[I-(1-p_c+\epsilon) \cdot T_{M0}]^{-1} - I]]
\end{split}
\end{equation}
\\
As we can see the difference lay in the ratio between the n-steps matrix and the original one. For $\epsilon \ll p_c$ we get : 
\begin{equation}
\begin{split}
    T_{M_0} \cdot (p_c+\epsilon) \cdot [I +[[I-(1-p_c+\epsilon) \cdot T_{M0}]^{-1} - I]]
\\= T_{M_0} \cdot (p_c+\epsilon) \cdot [I-(1-p_c+\epsilon) \cdot T_{M0}]^{-1}
\end{split}
\end{equation}
Now lets investigate the robustness of our analytic reconstruction by evaluating the expected results given the new data distribution :
\begin{equation}
    \begin{split}
        T_{reconstructed} = [I \cdot p_c + (1-p_c) \cdot T_{\hat{P}}]^{-1} \cdot T_{\hat{P}} \equiv  [I \cdot (1-p_c) + p_c \cdot T_{\hat{P}}^{-1}]^{-1} 
        \\ = [I \cdot (1-p_c) + \frac{p_c}{p_c + \epsilon} \cdot T_{M_0}^{-1} \cdot [I-(1-p_c+\epsilon)] \cdot T_{M_0}]]^{-1}
        \\ = [I \cdot (1-p_c) + \frac{p_c}{p_c + \epsilon} \cdot [T_{M_0}^{-1}-(1-p_c+\epsilon) \cdot I]]^{-1}
        \\ = [I \cdot \frac{\epsilon \cdot (1-2 \cdot p_c)}{p_c+ \epsilon} + \frac{p_c}{p_c+ \epsilon} \cdot T_{M_0}^{-1}]^{-1}
    \end{split}
\end{equation}
We can see that for $\epsilon \ll (1-p_c),p_c$ $T_{reconstruction} \approx T_{M_0}$.

\section{Evaluations Models}\label{Supp:EvaluationsModel}
In experiment \ref{fig:robuNonConsPC} the values of $p_c$ vary between sentences, hence, as $\sigma$ increase, the \emph{Naive} results present the relative effect of "fairly known" sentences in comparison to "badly known" ones. As all cases show, the benefit of "good" sentences is bigger than the disadvantage of "bad" ones. That is, non-constant $p_c$ is an advantage to the \emph{Naive} algorithm, but not for ours, resulting in diminishing advantage.
Notice that on closer look, not all cases are affected equally, especially, in the "Part-Of-Speech"(POS) case the difference is more noticeable. We believe that the reasons for this are: 1) relatively short distance of the POS transition matrix from its stationary distribution. 2) similarity between each entry in the transition matrix to the stationary distribution. That is, not only $T$ is similar to $T^d \xrightarrow[\infty]{d} \pi$ the fixed stationary distribution, it converges to $\pi$ for smaller d's. So, the effect of increasing d (i.e increasing $p_c$) is diminishing faster. Figure \ref{fig:supp:distStats} in the supplementary material compares the distance from stationarity for the different cases.  

Figure \ref{fig:supp:transHeat} presents heat-maps of the transitions matrices for the four models used for evaluation.
\begin{figure}[ht]
\vskip 0.2in
\begin{center}
    \includegraphics[scale=0.3]{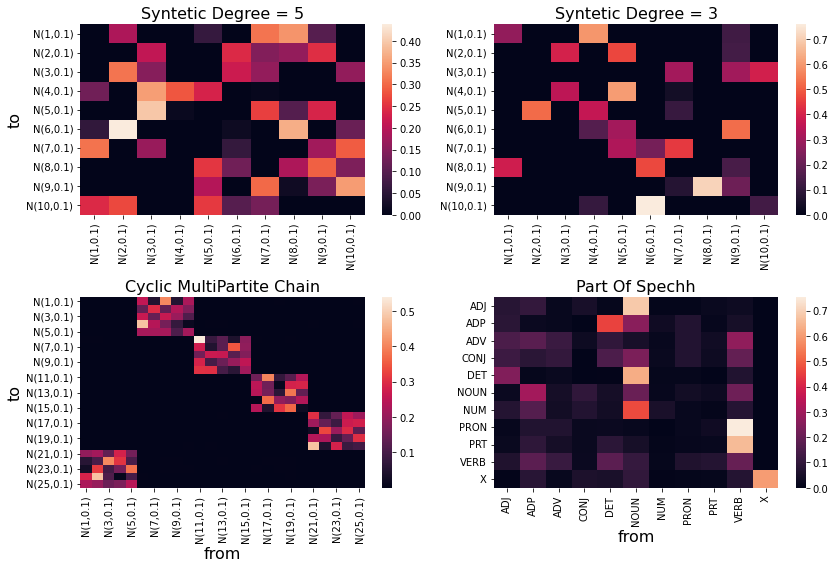}
    \caption{Heat maps of the transitions matrices of each model used in the paper.}
    \label{fig:supp:transHeat}
\end{center}
\vskip -0.2in
\end{figure}

Figure \ref{fig:supp:distStats} presents the distance to stationary distribution for each model .
\begin{figure}[ht]
\vskip 0.2in
\begin{center}
    \includegraphics[scale=0.3]{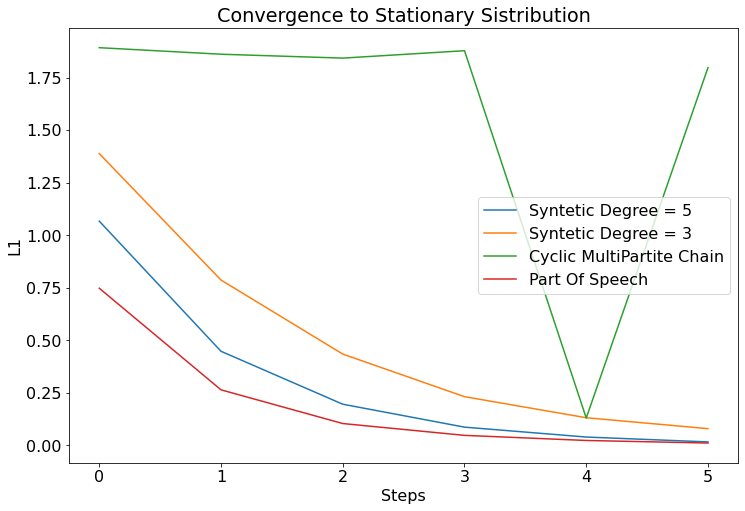}
    \caption{Distance to stationary distribution for each model used in the paper. X-label is the the number of steps for the d-step matrix $T^d$. Y-axis is the L1 distance between $T^d$ and $T$}
    \label{fig:supp:distStats}
\end{center}
\vskip -0.2in
\end{figure}

\section{Standard Deviation (S.D) of the Experiments Results}\label{Supp:SD}
For each figure in the paper, all the algorithms presented in the figure are evaluated on the same exact data and for the same random seed. The standard deviations(s.d) reported here are calculated as follow: 
for each figure (i.e "Known Specifications", "Wrong $p_c$", etc.), for each data model (i.e "Synthetic D=5", "Part Of Speech", etc. ), we picked one representative algorithm (most relevant to the figure) and parameters (most challenging), and reported the s.d between different data and seeds. 

All the s.d evaluations are presented in Table \ref{Table:SD}.

\begin{table}[t]
    \centering
    \caption{S.D for paper experiments}
    \label{Table:SD}
    \begin{tabular}{c|c}
\begin{tabular}{lllllr} 
\toprule
                    Model &                  Experiment & Algorithm & Algorithm Input &      Real Parameter &    S.D \\
             Syntetic D=5 & Ignore Missing Observations &     Naive &      $\emptyset$&         $\emptyset$ & 0.0604 \\
             Syntetic D=5 &                   Known $W$ &       HMM &             $W$ &                 $W$ & 0.0048 \\
             Syntetic D=5 &                 Known $p_c$ &     FOHMM &     $p_c = 0.5$ &         $p_c = 0.5$ & 0.0048 \\
             Syntetic D=5 &                   Known $N$ &     FOHMM &     $N   = 0.5$ &            $N = 80$ & 0.0231 \\
             Syntetic D=5 &                 Wrong $p_c$ &     FOHMM &     $p_c = 0.7$ &         $p_c = 0.5$ & 0.0117 \\
             Syntetic D=5 &          Non-constent $p_c$ &     FOHMM &     $p_c = 0.5$ & $p_c \sim N(.5,.2)$ & 0.0084 \\
             Syntetic D=5 &               Non-iid $p_c$ &     FOHMM &     $p_c = 0.5$ &   $\epsilon = 0.15$ & 0.0276 \\
             Syntetic D=3 & Ignore Missing Observations &     Naive & $\emptyset$     &   $\emptyset$       & 0.0450 \\
             Syntetic D=3 &                   Known $W$ &       HMM &             $W$ &                 $W$ & 0.0285 \\
             Syntetic D=3 &                 Known $p_c$ &     FOHMM &     $p_c = 0.5$ &         $p_c = 0.5$ & 0.0481 \\
             Syntetic D=3 &                   Known $N$ &     FOHMM &     $N   = 0.5$ &            $N = 80$ & 0.0346 \\
             Syntetic D=3 &                 Wrong $p_c$ &     FOHMM &     $p_c = 0.7$ &         $p_c = 0.5$ & 0.0599 \\
             Syntetic D=3 &          Non-constent $p_c$ &     FOHMM &     $p_c = 0.5$ & $p_c \sim N(.5,.2)$ & 0.0509 \\
             Syntetic D=3 &               Non-iid $p_c$ &     FOHMM &     $p_c = 0.5$ &   $\epsilon = 0.15$ & 0.1302 \\
                      POS & Ignore Missing Observations &     Naive &  $\emptyset$    &    $\emptyset$      & 0.0423 \\
                      POS &                   Known $W$ &       HMM &             $W$ &                 $W$ & 0.0162 \\
                      POS &                 Known $p_c$ &     FOHMM &     $p_c = 0.5$ &         $p_c = 0.5$ & 0.0079 \\
                      POS &                   Known $N$ &     FOHMM &     $N   = 0.5$ &            $N = 80$ & 0.0107 \\
                      POS &                 Wrong $p_c$ &     FOHMM &     $p_c = 0.7$ &         $p_c = 0.5$ & 0.0089 \\
                      POS &          Non-constent $p_c$ &     FOHMM &     $p_c = 0.5$ & $p_c \sim$ N(.5,.2) & 0.0091 \\
                      POS &               Non-iid $p_c$ &     FOHMM &     $p_c = 0.5$ &   $\epsilon = 0.15$ & 0.0116 \\
Cyclic MultiPartite Chain & Ignore Missing Observations &     Naive & $\emptyset$     & $\emptyset$         & 0.0094 \\
Cyclic MultiPartite Chain &                   Known $W$ &       HMM &             $W$ &                 $W$ & 0.0083 \\
Cyclic MultiPartite Chain &                 Known $p_c$ &     FOHMM &     $p_c = 0.5$ &         $p_c = 0.5$ & 0.0091 \\
Cyclic MultiPartite Chain &                   Known $N$ &     FOHMM &     $N   = 0.5$ &            $N = 80$ & 0.0144 \\
Cyclic MultiPartite Chain &                 Wrong $p_c$ &     FOHMM &     $p_c = 0.7$ &         $p_c = 0.5$ & 0.0067 \\
Cyclic MultiPartite Chain &          Non-constent $p_c$ &     FOHMM &     $p_c = 0.5$ & $p_c \sim$ N(.5,.2) & 0.0072 \\
Cyclic MultiPartite Chain &               Non-iid $p_c$ &     FOHMM &     $p_c = 0.5$ &   $\epsilon = 0.15$ & 0.0051 \\
\bottomrule
\end{tabular}
    \end{tabular}
\end{table}
\end{document}